\documentclass[accepted]{article}
\usepackage{aistats2022}
\usepackage[T1]{fontenc}    
\usepackage{hyperref}       
\usepackage{url}            
\usepackage{booktabs}       
\usepackage{amsfonts}       
\usepackage{nicefrac}       
\usepackage{microtype}      
\usepackage{lipsum}
\usepackage{graphicx}
\usepackage{mathtools, algorithm, algpseudocode}
\usepackage{amsthm}
\usepackage[document]{ragged2e}
\usepackage{enumitem}
\usepackage{bbm}
\usepackage{amssymb}
\usepackage{mathrsfs}
\usepackage{chngcntr}
\usepackage{bigints}
\usepackage{apptools}
\usepackage{subcaption}
\usepackage{caption}
\usepackage{graphicx}
\usepackage{enumitem}
\usepackage{comment}
\usepackage{color}
\usepackage{float}
\usepackage[english]{babel}
\usepackage[autostyle, english = american]{csquotes}
\usepackage[round]{natbib}
\AtAppendix{\counterwithin{theorem}{section}}
\MakeOuterQuote{"}
\theoremstyle{definition}
\newtheorem{definition}{Definition}[section]
\newtheorem{theorem}{Theorem}
\newtheorem{assumption}{Assumption}

\newtheorem*{remark}{Remark}
\newtheorem{lemma}[theorem]{Lemma}
\newtheorem{corollary}{Corollary}[theorem]
\begin{document}

%

%
\twocolumn[

\aistatstitle{Sobolev Norm Learning Rates for Conditional Mean Embeddings}

\aistatsauthor{Prem Talwai \And Ali Shameli \And  David Simchi-Levi }

\aistatsaddress{ORC, MIT \And IDSS, MIT \And IDSS, MIT} ]

\begin{abstract}
We develop novel learning rates for conditional mean embeddings by applying the theory of interpolation for reproducing kernel Hilbert spaces (RKHS). We derive explicit, adaptive convergence rates for the sample estimator under the misspecifed setting, where the target operator is not Hilbert-Schmidt or bounded with respect to the input/output RKHSs. We demonstrate that in certain parameter regimes, we can achieve uniform convergence rates in the output RKHS. We hope our analyses will allow the much broader application of conditional mean embeddings to more complex ML/RL settings involving infinite dimensional RKHSs and continuous state spaces.
\end{abstract}

\section{INTRODUCTION}
In the past decade, several studies have explored a new framework for embedding conditional distributions in reproducing kernel Hilbert spaces (RKHS). This approach seeks to represent a conditional distribution as an RKHS element, and thereby reduce the computation of conditional expectations to the evaluation of kernel inner products. Unlike other distribution learning approaches, which often involve density estimation and expensive numerical analysis, the conditional mean embedding (CME) framework exploits the popular kernel trick to allow distributions to be learned directly and efficiently from sample information, and do not require the target distribution to possess a density function. The broad generalizability and computational levity of conditional embeddings have led them to find many applications in reinforcement learning, hypothesis testing, and nonparametric inference \citep{fukumizu2007kernel, fukumizu2009kernel, grunewalder2012modelling, song2010nonparametric}, where conditional relationships are often of pertinent interest. 

A central issue involved in the conditional embedding framework is the performance of the sample estimator. Despite their successful application, there has been a limited study of optimal learning rates for conditional mean embeddings. Several foundational works \citep{song2010nonparametric, song2009hilbert} established the consistency of the sample embedding estimator, exploring its convergence rate to a ``true'' embedding in the RKHS norm. These works framed the act of conditioning  as a linear operator between two Hilbert spaces, which mapped features of the independent variable in the input space to the mean embeddings of their respective conditional distributions in the output feature space. Under certain smoothness conditions on the underlying distribution,  \cite{song2010nonparametric} demonstrated convergence of the sample estimator in the Hilbert-Schmidt norm. Although these works introduced a regularization parameter to tackle the ill-conditioning of the sample covariance operator, the learning task was not explicitly framed as a regularized regression problem, with the consistency of the sample estimator  only implicitly depending on the polynomial decay of the regularizer. Later, \cite{grunewalder2012conditional} explicitly formulated conditional embeddings as the solution of a vector-valued Tikhonov-regularized regression problem. Here, the learning target was framed as a Hilbert space-valued function acting directly on the independent variable. Drawing from the rich theory of regularized regression \citep{caponnetto2007optimal}, they derived near-optimal learning rates for kernels whose spectrum exhibits polynomial decay. However, their analysis required the compactness of the input set and the target Hilbert space to be finite-dimensional, an assumption which is violated by several common kernels. 

In recent years, there have been several attempts to further relax the hypotheses of the previous two approaches — namely the requirement of a finite-dimensional output RKHS and the compactness of the true conditional mean operator (well-specification). These approaches have sought a measure-theoretic interpretation of conditional mean embeddings as Hilbert-valued Bochner-measurable random variables in either an operator or vector-valued RKHS. \cite{klebanov2020rigorous} demonstrates almost sure consistency for centered operators under relatively weak assumptions, but only $L^2$ consistency in the more popular uncentered framework, providing no insight into learning rates in either case. \cite{park2020measure} abandons the operator framework, and seeks to directly extend the vector-valued regression approach from \cite{grunewalder2012conditional} to infinite-dimensional RKHS, but similarly only demonstrates consistency in the general setting, and must further assume the well-specified setting to provide an concrete learning rate for the surrogate risk. Moreover, the latter approach sacrifices the operator interpretation of the conditional embedding, which has recently found an elegant connection to transfer operators in dynamical systems theory, and their associated data-driven spectral techniques \citep{mollenhauer2020kernel, klus2018data}.

In this paper, we aim to address these gaps by deriving novel adaptive learning rates for conditional mean embeddings in the \textit{misspecified} setting, that elucidate the relationship between the properties of the kernel class and target measure. In particular, we seek to capture the interplay between kernel complexity (as measured by eigenvalue decay/summability) and the continuity of the hypothesis class in establishing uniform convergence rates. We apply the theory of interpolation spaces for RKHS \citep{fischer2020sobolev, steinwart2012mercer} to significantly relax the aforementioned source conditions, and simply require that the target "conditioning function'' lie in some intermediate fractional space between the input RKHS and $L^2$. To the best of our knowledge, this is the first work to establish uniform convergence rates in the misspecified setting. Our approach is also distinct from existing operator-based methods in that we do not require the the target conditional mean operator to be Hilbert-Schmidt, but simply bounded on the aforementioned interpolation space. Our generalized notion of boundedness also significantly attenuates the need to explicitly verify this continuity condition, which can often be difficult and unintuitive, and was a key motivator of the regression approach adapted by \cite{grunewalder2012conditional, park2020measure}. Moreover, our analysis does not make any assumptions on the dimensionality of the input/output RKHS or the compactness of the latent spaces. These relaxations do introduce a slight tradeoff of requiring the polynomial eigendecay of the covariance operator, a standard assumption in regularized least-squares problems \citep{lin12020optimal, lin2020optimal, caponnetto2007optimal}. In a sense, our approach hybridizes the two aforementioned frameworks --- namely, like \cite{song2010nonparametric} we construct conditional embeddings as operators, and characterize the convergence of the sample estimator via the spectral structure of the target embedding operator. However, we seek inspiration from the regression formulation of \cite{grunewalder2012conditional} and likewise try to extrapolate approaches from scalar-valued kernel regression to our operator learning problem.

\section{MODEL AND PRELIMINARIES} \label{Model and Preliminaries}
\subsection{Problem Statement}
Let $\mathcal{D} = \{(x_i, y_i)\}_{i = 1}^n \subset \mathcal{X} \times \mathcal{Y}$ be a dataset of $n$ independent, identically distributed observations sampled from a distribution $P$. Our goal is to learn the conditional distribution $P(Y|X)$, where $(X, Y) \sim P$. Here, we study a learning strategy based on conditional mean embeddings \citep{song2010nonparametric}, which seek to represent conditional distributions as operators between an input and output RKHS. Formally, let $\mathcal{H}_{K}$ be a separable RKHS on $\mathcal{X}$ with bounded measurable kernel $k(\cdot,  \cdot)$ and $\mathcal{H}_{L}$ be a separable RKHS on $\mathcal{Y}$ with measurable kernel $l(\cdot, \cdot)$. Then, according to \cite{song2009hilbert}, we define a conditional mean embedding $C_{Y|X}: \mathcal{H}_K  \to \mathcal{H}_L$  as follows:
\begin{definition}
\label{CME Def}
The conditional mean embedding operator $C_{Y|X}: \mathcal{H}_K  \to \mathcal{H}_L$ is defined such that:
\begin{itemize}
    \item $\mu_{Y|x} \equiv \mathbb{E}_{Y|x}[l(Y, \cdot)] = C_{Y|X}k(x, \cdot)$
    \item $\mathbb{E}_{Y|x}[g(\cdot)] = \langle g, \mu_{Y|x} \rangle_{L}$ for all $g \in \mathcal{H}_L$ and $x \in \mathcal{X}$
\end{itemize}
\end{definition}
\par Essentially, the operator $C_{Y|X}$ performs the action of conditioning on some $x \in \mathcal{X}$, which is represented by its feature mapping $k(x, \cdot) \in \mathcal{H}_K$. The output $\mu_{Y|x} \in \mathcal{H}_L$ then represents the conditional distribution $P(\cdot|x)$ in the output feature space $\mathcal{H}_L$--- evaluating the conditional expectation of some function $g \in \mathcal{H}_L$ simply reduces to taking its inner product with $\mu_{Y|x}$. Thus, in a sense, $\mu_{Y|x}$ can be interpreted as a generalization of the ``density" of $P(\cdot | x)$, although it is important to note that distributions do not need to possess Lebesgue densities to be represented via a CME. 
\par It is important to note that, implicit in Definition \ref{CME Def} is the assumption that the function $g_f(\cdot) = \mathbb{E}[f(Y) | X = \cdot]$ is contained in $\mathcal{H}_K$ for every $f \in \mathcal{H}_L$. This is a strong assumption, and forms the so-called ``well-specified'' scenario treated exhaustively in the literature (see e.g \cite{song2009hilbert,song2010nonparametric}). It is violated in several common cases, such as when $X$ and $Y$ are independent and $\mathcal{H}_K$ is a Gaussian RKHS, which does not contain the constant functions $g_f(\cdot)$ for any $f \in \mathcal{H}_L$ (see Corollary 4.44 in \cite{steinwart2008support}; our Lemma \ref{Gaussian Interpolation Spaces include Constants} demonstrates that constants \textit{are} included in every interpolation space, however). A key feature of our analysis will involve relaxing this assumption by replacing $\mathcal{H}_K$ in Definition \ref{CME Def} with a larger \textit{interpolation} space $\mathcal{H}^{\beta}_K$ that lies ``between'' $\mathcal{H}_K$ and $L^2(P_X)$ (defined rigorously in the following section). Hence, our framework proves robust as long there exists some such fractional space that contains $g_f(\cdot)$ for every $f \in \mathcal{H}_L$ --- in section \ref{Contributions}, we demonstrate how our learning rates depend on the smoothness of this space.


\par We also define the uncentered kernel covariance $C_{XX} = \mathbb{E}_{X}[k(X, \cdot) \otimes k(X,  \cdot)]$ and cross-covariance  $C_{YX} = \mathbb{E}_{YX}[l(Y, \cdot) \otimes k(X,  \cdot)]$ operators. Note here that $\otimes$ may be interpreted as a tensor product, i.e. $C_{XX}$, for example, may be alternatively expressed as: $C_{XX} = \mathbb{E}_{X}[k(X, \cdot)\langle k(X, \cdot), \cdot \rangle_{K}]$, if we wish to make the action of $C_{XX}$ on $\mathcal{H}_K$ more explicit. It can be easily shown \citep{klebanov2020rigorous} that $C_{Y|X} = (C^{\dagger}_{XX}C_{XY})^{*}$, when $C_{Y|X}$ exists (where $\dagger$ denotes the pseudo-inverse and $*$ the adjoint). 

\par In practice, we do not have access to the true covariance operators $C_{XX}$ and $C_{YX}$, and hence use the regularized sample CME $\hat{C}^{\lambda}_{Y|X} = \hat{C}_{YX}(\hat{C}_{XX} + \lambda)^{-1}$, where $\lambda > 0$ and the empirical operators $\hat{C}_{YX}$ and $\hat{C}_{YX}$ are defined precisely like their population counterparts (with $\mathbb{E}_{YX}[\cdot]$ replaced by the empirical expectation $\mathbb{E}_{\mathcal{D}}[\cdot]$). \cite{grunewalder2012conditional} demonstrated that $\hat{C}^{\lambda}_{Y|X}$ solves the following regularized least-squares problem:
\begin{small}
\begin{equation}
\label{Regression Problem}
    \text{arg} \min_{T: \mathcal{H}_{K} \to \mathcal{H}_L} \frac{1}{n}\sum_{i = 1}^n ||l(y_i, \cdot) - T[k(x_i, \cdot)]||^2_{L} + \lambda||T||^2_{\text{HS}}
\end{equation}
\end{small}

Traditionally, the sample complexity of solutions to \eqref{Regression Problem} has been analyzed through the lens of vector-valued regression (e.g. \cite{park2020measure, grunewalder2012conditional}). In earlier works, the consistency of the sample CME was demonstrated via a spectral characterization \citep{song2010nonparametric} that imposed strong compactness conditions on $C_{Y|X}$. Here, we develop an integral operator approach towards the analysis of \eqref{Regression Problem} that seeks to significantly weaken the spectral conditions on $C_{Y|X}$ through the use of interpolation spaces --- our  approach is strongly motivated by \cite{fischer2020sobolev} where integral operator techniques were successfully applied towards the analysis of scalar-valued kernel regression problems. In section \ref{Contributions}, we notably demonstrate that we can achieve the same learning rates derived in \cite{fischer2020sobolev} for our operator regression problem, under weaker smoothness conditions. 

\par 

\begin{remark}[Proofs]
All proofs can be found in the supplementary appendices.
\end{remark}
\begin{remark}[Notation]
In the remainder of this paper, we define $\hat{C}_{Y|X} \equiv \hat{C}_{YX}(\hat{C}_{XX} + \lambda I)^{-1}$, $C^{\lambda}_{Y|X} \equiv C_{YX}(C_{XX} + \lambda I)^{-1}$, and $\mu_{Y|x} = \mathbb{E}_{Y|X = x}[l(Y, \cdot)]$, $\hat{\mu}_{Y|x} = \hat{C}_{Y|X}(k(x, \cdot))$, and $\mu^{\lambda}_{Y|x} = C^{\lambda}_{Y | X}(k(x, \cdot))$. Note that when denoting the sample conditional embedding  $\hat{C}_{Y|X}$ we suppress the dependence on $\lambda$ and the number of samples $n$, as these are typically understood from context. Moreover, for any two Banach spaces $A$ and $B$, we denote by $\mathcal{L}(A, B)$ the set of all continuous linear operators between $A$ and $B$. For any $T \in \mathcal{L}(A, B)$, $||T||$ denotes the operator norm given by $||T|| = \sup_{||x||_{A} \leq 1} ||Tx||_{B}$ and $||T||_{\text{HS}}$ denotes a Hilbert-Schmidt norm. Occasionally, we denote this operator norm as $||\cdot||_{A \to B}$ in order to make the domain and codomain more explicit. Finally, we use the symbol $\preccurlyeq$ to denote the Loewner (semidefinite) order between positive semidefinite operators (i.e. $A \preccurlyeq B$ iff $B - A$ is positive semidefinite). 
\end{remark}

\subsection{Mathematical Preliminaries}
\label{Preliminaries}

We first summarize the theory of interpolation spaces between $\mathcal{H}_{K}$ and $\mathcal{L}^2(\nu)$ (where $\nu = P_X$ is the marginal distribution on $\mathcal{X}$). Consider the injective imbedding $I_{\nu}: \mathcal{H}_K \to L^2(\nu)$ of $\mathcal{H}_{K}$ into $\mathcal{L}^2(\nu)$. Let $S_{\nu} = I^{*}_{\nu}$ be its adjoint. Then, it can be shown that $S_{\nu}$ is an integral operator given by:
\begin{equation*}
    S_{\nu}(f) = \int_{\mathcal{X}} k(x, \cdot)f(y)d\nu(y)
\end{equation*}
Using $S_{\nu}$ and $I_{\nu}$, we construct the following positive self-adjoint operators on $\mathcal{H}_K$ and $\mathcal{L}^2(\nu)$, respectively:
\begin{align*}
    C_{\nu} & = S_{\nu}I_{\nu} =  I^{*}_{\nu}I_{\nu} \\
    T_{\nu} & = I_{\nu}S_{\nu} =  I_{\nu} I^{*}_{\nu}
\end{align*}
We observe that $C_{\nu}$ and $T_{\nu}$ are nuclear (see Lemma 2.2/2.3 in \cite{steinwart2012mercer}), and that $C_{\nu}$ coincides with our uncentered cross-covariance operator $C_{XX}$. In our discussion/analyses below we typically only use the notation $C_{XX}$ when considering expansions of the operators $\hat{C}_{Y|X}, C_{Y|X}$, or $C^{\lambda}_{Y|X}$ in order to remain consistent with literature (when the latter operators are abbreviated as in this sentence, we instead use $C_{\nu}$). Since, $T_{\nu}$ is nuclear and self-adjoint, it admits a spectral representation:
\begin{equation*}
    T_{\nu} = \sum_{j = 1}^{\infty} \mu_j e_j \langle e_j, \cdot \rangle_{L^2(\nu)}
\end{equation*}
where $\{\mu_j\}_{j = 1}^{\infty} \in (0, \infty)$ are nonzero eigenvalues of $T_{\nu}$ (ordered nonincreasingly) and $\{e_j\}_{j = 1}^{\infty} \subset L^2(\nu)$ form an orthonormal system of corresponding eigenfunctions. Note that formally, the elements $e_j$ of $L^2(\nu)$ are equivalence classes $[e_j]_{\nu}$ whose members only differ on a set of $\nu$-measure zero--- notationally, we consider this formalism to be understood here and simply write $e_j$ to refer to elements in both $\mathcal{H}_K$,  $L^2(\nu)$, and their interpolation spaces (with the residence of $e_j$ understood from context). We define the interpolation spaces $\mathcal{H}^{\alpha}_K$ as:
\begin{definition}
For $\alpha > 0$, we define the space $\mathcal{H}^{\alpha}_{K}$:
\begin{equation*}
    \mathcal{H}^{\alpha}_{K} = \Big\{f = \sum_{i} a_i (\mu_{i}^{\frac{\alpha}{2}}e_i): \{a_i\}_{i = 1}^{\infty} \in \ell^2\Big\}
\end{equation*}
\end{definition}
with inner product:
\begin{equation*}
    \Big \langle \sum_{i} a_i (\mu_i^{\frac{\alpha}{2}}e_i), \sum_{i} b_i (\mu_i^{\frac{\alpha}{2}}e_i) \Big \rangle_{\mathcal{H}^{\alpha}_{K}} = \sum_{i} a_i b_i
\end{equation*}
We observe that, if $\alpha > \beta$, $\mathcal{H}^{\alpha}_{K} \subset \mathcal{H}^{\beta}_{K} \subset L^2(\nu)$, with $\mathcal{H}_K^{1} = \mathcal{H}_{K}$. Note it is easy to see that $\{\mu_i^{\frac{\alpha}{2}}e_i\}_{i=1}^{\infty}$ is an orthonormal basis for $\mathcal{H}^{\alpha}_{K}$. We also observe that if:
\begin{equation}
\label{Interpolation is RKHS}
    \sum_{i = 1}^{\infty} \mu_i^{\alpha}e^2_i(x) < \infty \hspace{2mm} \forall \hspace{1mm} x \in \mathcal{X} 
\end{equation}
then $\mathcal{H}^{\alpha}_{K}$ can be viewed as an RKHS whose reproducing kernel is equivalent to that of the integral operator $T_{\nu}^{\alpha}$ on $L^2(\nu)$ (Proposition 4.2 in \cite{steinwart2012mercer}). Even, when \eqref{Interpolation is RKHS} is not satisfied, we denote this kernel as $k^{\alpha}(x, y) = \sum_{i} \mu^{\alpha}_i e_i(x)e_i(y)$, and write $||k^{\alpha}||_{\infty} =  \sup_{x \in \mathcal{X}} \sum_{i = 1}^{\infty} \mu_i^{\alpha}e^2_i(x)$, if the latter quantity is finite. Hence, we may identify $\mathcal{H}^{\alpha}_{K} \cong \overline{\text{ran} \hspace{1mm} T_{\nu}^{\frac{\alpha}{2}}}$. A detailed development of RKHS interpolation spaces can be found in \cite{steinwart2012mercer}.

\subsection{Conditional Embeddings on Interpolation Spaces}
\label{CME Interpolation}
We develop the notion of conditional embeddings on interpolation spaces.  We begin with the following definition:
\begin{definition}
\label{Sobolev Norm}
Let $T: \mathcal{H}^{\beta}_K \to \mathcal{H}_L$ be a (possibly unbounded) operator for some $\beta > 0$ and let $\gamma \in (0, \beta)$. Let $I_{\beta, \gamma, \nu}: \mathcal{H}^{\beta} \to \mathcal{H}^{\gamma}$ be the canonical embedding. We define the operator norm $||\cdot||_{\gamma}$:
\begin{equation*}
\label{Sobolev Norm Formula}
    ||T||_{\beta, \gamma} = ||T \circ I^{*}_{\beta, \gamma, \nu}||_{\mathcal{H}^{\gamma}_{K} \to \mathcal{H}_L}
\end{equation*}
where both norms may possibly be infinite. When $\beta = 1$, we simply write $||\cdot||_{\gamma}$
\end{definition}

Our definition of the interpolation norm is motivated by the following observation:
\begin{lemma}
\label{Equivalence of Mean Embeddings}
Suppose $C_{Y|X}: \mathcal{H}_K \to \mathcal{H}_L$ is well-defined according to Definition \ref{CME Def}. Then, for any $\beta \in (0, 1)$, $C_{Y|X} \circ I^{*}_{1, \beta, \nu}$ is the conditional mean embedding from $\mathcal{H}^{\beta}_K$ to $\mathcal{H}_L$ (by Definition \ref{CME Def} with $\mathcal{H}_{K}$ replaced by $\mathcal{H}^{\beta}_{K}$).
\end{lemma}
When the conditional mean embedding from $\mathcal{H}^{\beta}_{K}$ to $\mathcal{H}_L$ is well-defined, we denote it as $C^{\beta}_{Y|X}$. Note, implicit in this definition of $C^{\beta}_{Y|X}$ is the assumption that $\mathcal{H}^{\beta}_K$ is indeed an RKHS, i.e. it satisfies condition \eqref{Interpolation is RKHS}.  Thus, from Lemma \ref{Equivalence of Mean Embeddings}, we observe that if $C_{Y|X}$ and $C^{\beta}_{Y|X}$  are well-defined, then $||C_{Y|X}||_{\beta} = ||C^{\beta}_{Y|X}||$. The following result further elaborates the relationship between the operator norms $||\cdot||$ and $||\cdot||_{\gamma}$:
\begin{lemma}
\label{Operator Norms in Interpolation Spaces}
Let $T: \mathcal{H}^{\beta}_K \to \mathcal{H}_L$ be an operator. Then, for any $\gamma \in (0, \beta)$, we have that:
\begin{equation*}
    ||T||_{\beta, \gamma} = ||T \circ C_{\beta, \gamma, \nu}^{\frac{1}{2}}||_{\mathcal{H}^{\beta}_K \to \mathcal{H}_L}
\end{equation*}
where $C_{\beta, \gamma, \nu} = I^{*}_{\beta, \gamma, \nu}I_{\beta, \gamma, \nu}$ 
\end{lemma}
Our motivation behind introducing the Sobolev norms in Definition \ref{Sobolev Norm} stems from our desire to study operator convergence over the interpolation spaces $\mathcal{H}^{\beta}_K$. A distinguishing feature of our analysis is that we do not assume the existence of the CME $C_{Y|X}$ over $\mathcal{H}_K$, but merely over some interpolant $C^{\beta}_{Y|X}$ ($\beta \in (0, 2)$), which maps $k^{\beta}(x, \cdot) \in H^{\beta}_{K}$ to $\mu_{Y|x}$ (we are primarily interested in the misspecified setting  $0 < \beta < 1$).  Since, we cannot approximate $C^{\beta}_{Y|X}$ directly (as the exponent $\beta$ is typically unknown), we construct the regularized approximation $C^{\lambda}_{Y|X} \in \mathcal{L}(\mathcal{H}_K, \mathcal{H}_L)$ (which is always well-defined and bounded) and ``pushback'' to $\mathcal{H}^{\beta}_K$ via the composition $C^{\lambda}_{Y|X} \circ I^{*}_{1, \beta, \nu}$. We observe that $(C^{\lambda}_{Y|X} \circ I^{*}_{1, \beta, \nu}) k^{\beta}(x, \cdot) = C^{\lambda}_{Y|X}k(x, \cdot) = \mu^{\lambda}_{Y|x}$, i.e. $C^{\lambda}_{Y|X} \circ I^{*}_{1, \beta, \nu}$ maps the canonical ``feature'' $k^{\beta}(x, \cdot)$ in $\mathcal{H}^{\beta}_K$ to the regularized mean embedding $\mu^{\lambda}_{Y|x} \in \mathcal{H}_L$. 

Thus, the use of Sobolev norms here is primarily a \textit{mathematical construction employed to compare operators defined over different domains} --- in applications, we are mainly interested in estimating $||\mu_{Y|x} - \hat{\mu}_{Y|x}||_{L}$, i.e. the distance between the sample and true embeddings of the conditional distribution $P(\cdot | x)$ in the output RKHS $\mathcal{H}_L$. Bounding the latter distance provides insight into the sample error involved in computing the conditional expectation of a function $g \in \mathcal{H}_L$, as 
\begin{align*}
    |\langle g, \hat{\mu}_{Y|x} \rangle_L - \mathbb{E}[g(Y)|x]| & = |\langle g, \hat{\mu}_{Y|x} \rangle_L -  \langle g, \mu_{Y|x} \rangle_L| \\
    & \leq ||\mu_{Y|x} - \hat{\mu}_{Y|x}||_{L}||g||_{L}
\end{align*}
(note that $\langle g, \hat{\mu}_{Y|x} \rangle_L$ is typically \textit{not} an expectation of $g$ with respect to some distribution, but simply an approximation of the true expectation $\mathbb{E}[g(Y)|x]$; see \cite{grunewalder2012modelling} for more details). If $\mathcal{H}^{\beta}_{K}$ is continuously embedded in $L^{\infty}(\mathcal{X})$ (i.e. $k^{\beta}$ is bounded), then we can obtain uniform bounds on $||\mu_{Y|x} - \hat{\mu}_{Y|x}||_{L}$ over all $x \in \mathcal{X}$, by estimating the operator distance $||C^{\lambda}_{Y|X} \circ I^{*}_{1, \beta, \nu} - C^{\beta}_{Y|x}||$. Indeed, we have:
\begin{small}
\begin{align*}
    ||\mu_{Y|x} - \hat{\mu}_{Y|x}||_{L} & = ||(\hat{C}_{Y|X} \circ I^{*}_{1, \beta, \nu})k^{\beta}(x, \cdot) - (C^{\beta}_{Y | X})k^{\beta}(x, \cdot)||_{L} \\
    & \leq ||\hat{C}_{Y|X} \circ I^{*}_{1, \beta, \nu} - C^{\beta}_{Y | X}|| ||k^{\beta}(x, \cdot)||_{\beta} \\
    & \leq ||k^{\beta}||_{\infty}||\hat{C}_{Y|X} \circ I^{*}_{1, \beta, \nu} - C^{\beta}_{Y | X}||
\end{align*}
\end{small}
In the following section, we discuss the various parameter regimes in which such bounds are attainable. 

\begin{remark}[Abuse of Notation]
In light of Lemma \ref{Equivalence of Mean Embeddings}, for the remainder of the paper, when $C^{\beta}_{Y|X}$ is well-defined, we abuse notation and simply write $||\hat{C}_{Y|X}  - C_{Y | X}||_{\beta}$ to express $||\hat{C}_{Y|X} \circ I^{*}_{1, \beta, \nu} - C^{\beta}_{Y | X}||$, even when $C_{Y | X}$ is not well-defined/bounded, in order to make explicit the distance between a sample estimator and its ``true'' value. Similarly, for any $\gamma < \beta$, we define $||\hat{C}_{Y|X} - C_{Y|X}||_{\gamma}$ as $||\hat{C}_{Y|X} \circ I^{*}_{1, \gamma, \nu} - C^{\beta}_{Y|X} \circ I^{*}_{\beta, \gamma, \nu}||$
\end{remark}

\subsection{Assumptions}
\label{Assumptions Section}
We state some assumptions similar to those in \cite{fischer2020sobolev} --- namely, we impose conditions on the decay of the eigenvalues of $T_{\nu}$, the boundedness of a kernel interpolant, the conditional kernel moments of our target distribution, and the boundedness of $C^{\beta}_{Y|X}$. Below, we discuss how the latter assumption is weaker than the direct generalization of its analogous hypothesis in \cite{fischer2020sobolev} for scalar-valued regression.
\begin{assumption}
\label{EVD}
There exists a $0 < p < 1$ such that $ci^{-\frac{1}{p}} \leq \mu_i \leq Ci^{-\frac{1}{p}}$, for some $c, C > 0$
\end{assumption}
\begin{assumption}
\label{EMB}
There exists a $0 < p < \alpha \leq 1$ such that the inclusion map $i: H_K^{\alpha} \hookrightarrow L^{\infty}(\nu)$ is continuous, with $||i|| = ||k^{\alpha}|| \leq A$ for some $A > 0$ (we define $\alpha$ as the smallest value satisfying these conditions)
\end{assumption}
\begin{assumption}
\label{Hypothesis Space}
There exists a $0 < p < \beta < 2$ such that $||C^{\beta}_{Y|X}|| \leq B < \infty$
\end{assumption}
\begin{assumption}
\label{Subexponential}
There exists a trace-class operator $V: \mathcal{H}_L \to \mathcal{H}_L$ and scalar $R > 0$, such that for every $x \in \mathcal{X}$ and $p \geq 1$:
\begin{small}
\begin{equation}
\label{Moment Condition}
    \mathbb{E}_{Y|x}\Big[\Big((L(Y, \cdot) - \mu_{Y|x}) \otimes (L(Y, \cdot) - \mu_{Y|x})\Big)^p\Big] 
    \preccurlyeq \frac{(2p)!R^{2p - 2}}{2}V
\end{equation}
\end{small}
\end{assumption}

\subsubsection{Discussion/Comparison of Assumptions} 
\label{Assumption Discussion Section}
Although they are listed separately here, assumptions \ref{EVD} and \ref{EMB} are indeed highly related as they both (implicitly) impose conditions on the summability of powers of eigenvalues of $T_{\nu}$. Indeed, under an additional assumption of uniform boundedness of the eigenfunctions $e_i$, assumptions \ref{EVD} and \ref{EMB} can be shown to be equivalent for certain domains of $\alpha$ and $p$. A comprehensive discussion of the relationship between these two assumptions can be found in \cite{fischer2020sobolev, steinwart2012mercer}. 

\par Assumption \ref{Hypothesis Space} characterizes the continuity of the ``true'' conditional embedding operator. Note, a distinctive feature of our approach is that we not only allow the CME to exist over any fractional RKHS $\mathcal{H}^{\beta}_K$, but additionally only require that the CME is \textit{bounded} over this space. This contrasts strongly with existing operator-theoretic literature \citep{song2009hilbert, song2010nonparametric} where the CME is required to be Hilbert-Schmidt (or equivalently belong to a product RKHS in the regression formulation of \cite{park2020measure}) in order to achieve explicit learning rates. The significance of this relaxation can be seen in the trivial example when $Y = X$ and $1 - \beta \leq \frac{p}{2}$: indeed here, it can be easily seen that $C^{\beta}_{Y|X} = I^{*}_{1, \beta, \nu}$, and hence $||C^{\beta}_{Y|X}||  = \mu^{1 - \beta}_1 < \infty$ while $||C^{\beta}_{Y|X}||^2_{\text{HS}} = \sum_{i = 1}^{\infty} \mu^{2(1 - \beta)}_i \geq \sum_{i = 1}^{\infty} i^{-1} = \infty$. More, generally, it can be shown that if $C^{\beta}_{Y|X}$ exists, then it is automatically bounded when $\mathbb{E}_{Y}[l(Y, Y)] < \infty$, i.e. when $l$ is bounded (see Lemma \ref{Characterization of Hypothesis Space} below). Note, however, compared to the scalar regression case in \cite{fischer2020sobolev}, this introduces the condition that $p < \beta$, which is trivially satisfied when $\mathcal{H}_{K}^{\beta}$ is an RKHS (Proposition 4.4 in \cite{steinwart2012mercer}). We are able to remove the latter condition if we require $C^{\beta}_{Y|X}$ to be Hilbert-Schmidt, which would be a direct generalization of the source condition in \cite{fischer2020sobolev} (this tradeoff is directly indicated in the remark following the proof of Lemma \ref{orders} in Appendix B). 

\par We are primarily interested in the misspecified/``hard learning'' scenario when $0 < \beta < 1$, the regime where the conditional embedding is \textit{not} bounded over the RKHS $\mathcal{H}_K$, as this is where our framework improves on the related literature. Note that since the interpolation spaces are descending, the regime with $1 \leq \beta \leq 2$ simply collapses to $\beta = 1$, which has already been analyzed in \cite{song2010nonparametric}. We only include this regime here, to demonstrate that we can generalize the learning rates from \cite{fischer2020sobolev} almost exactly. We now demonstrate the relationship between Assumption \ref{Hypothesis Space} and the hard-learning scenario:

\begin{lemma}
\label{Characterization of Hypothesis Space}
Assumption \ref{Hypothesis Space} is equivalent to $\sup_{||f||_{L} \leq 1} ||\mathbb{E}[f(Y)| X = \cdot]||_{H^{\beta}_K} = B < \infty$ for some $\beta$ with $0 < p < \beta < 2$. If $\mathbb{E}_{Y}[l(Y, Y)] < \infty$ and $C^{\beta}_{Y|X}$ exists, then Assumption \ref{Hypothesis Space} is automatic.
\end{lemma}

\par Lemma \ref{Characterization of Hypothesis Space} captures the generality of our approach --- notice that unlike the classical framework of CME, we do not require $\mathbb{E}[f(Y)| X = \cdot] \in \mathcal{H}_K$ for $f \in \mathcal{H}_L$. Indeed, by Lemma \ref{Characterization of Hypothesis Space}, $\mathbb{E}[f(Y)| X = \cdot] \in \mathcal{H}^{\beta}_K$ must only lie in a $||\cdot||_{\mathcal{H}^{\beta}_K}$ ball of radius $B$ for all unit vectors $f \in \mathcal{H}_{L}$. Intuitively, the act of conditioning on $\mathcal{X}$ must map $\mathcal{H}_{L}$ continuously into $\mathcal{H}^{\beta}_K$, which is strictly larger than $\mathcal{H}_{K}$ for $\beta \in (0, 1)$ --- the misspecified setting. 

\par Recall that in the previous section, we demonstrate that uniform convergence rates are attainable when $\alpha < \beta$ (note that Assumption \ref{EMB} automatically qualifies $\mathcal{H}^{\alpha}_K$ as a bounded RKHS). This setting is attainable in many common settings --- for example, when $k$ is a Mat\'{e}rn kernel of order $\gamma > 0$ on a bounded open subset $\mathcal{X} \subset \mathbb{R}^d$ with strong Lipschitz boundary, $k^{\alpha}$ is bounded for all $\alpha \in \Big(\frac{2\gamma + d}{d}, 1\Big)$ (see e.g. Example 4.8 in \cite{steinwart2019convergence}). Moreover, in this scenario, the condition $p < \beta$ translates to requiring $\mathcal{H}_{K}^{\beta} \cong W^{s}(\mathcal{X})$ for $s > \frac{d}{2}$.

\par Assumption \ref{Subexponential} may be viewed as an ``operator subexponential'' condition that controls the norm of the conditional operator MGF. Like  Assumption \ref{Hypothesis Space}, Assumption \ref{Subexponential} can be weakened to  $\mathbb{E}_{Y|x}[||l(Y, \cdot) - \mathbb{E}_{Y|x}[l(Y, \cdot)]||^{2p}] \leq \frac{(2p)!R^{2p - 2}}{2}\sigma^2$ (for $\sigma \in \mathbb{R}$) if Assumption \ref{Hypothesis Space} is replaced with a stronger Hilbert-Schmidt criterion (which would be the natural generalization of the corresponding assumptions from \cite{fischer2020sobolev} to operator-valued RKHS). However, Lemma \ref{Satisfying Moment Condition} demonstrates that Assumption \ref{Subexponential} is satisfied when the \textit{output} RKHS also satisfies a variant of Assumptions \ref{EVD}/\ref{EMB}, suggesting that the tradeoff we choose here indeed achieves more generality. 

\begin{lemma}
\label{Satisfying Moment Condition}
Let $\pi$ be a measure on $\mathcal{Y}$. Suppose $\mathcal{H}_{L}$ is compactly and injectively embedded in $L^2(\pi)$, $l$ is bounded ($\sup_{y \in \mathcal{Y}} \sqrt{l(y, y)} = \ell < \infty$), and $T_{\pi}$ has spectrum $\{(\eta_i, f_i)\}_{i = 1}^{\infty}$ (where $T_{\pi}$ is defined on $L^2(\pi)$ analogously to $T_{\nu}$ above). Then, if $\eta_i = \mathcal{O}\Big(i^{-q^{-1}}\Big)$ for $0 < q < 1$, and $K \equiv \sup_{y \in \mathcal{Y}} \sum_{i = 1}^{\infty} \eta^{\gamma}_if^2_i(y) < \infty$ for some $\gamma \in (0, 1-q)$, we have that Assumption 4 is satisfied for $R = 2\ell$ and $V = KC^{1 - \gamma}_{\pi}$.
\end{lemma}

Thus, Lemma \ref{Satisfying Moment Condition} demonstrates that we can reduce Assumption \ref{Subexponential} to a condition on the RKHS $\mathcal{H}_L$, instead of a constraint on the \textit{conditional distribution} $Y|x$. Moreover, although Assumptions \ref{EVD} and \ref{EMB} are more restrictive than the measure-theoretic frameworks of \cite{mollenhauer2020kernel} and \cite{park2020measure}, these hypotheses do not impose conditions on the conditional distribution $P(Y|X)$ or the CME, but instead prescribe the relationship between the kernel and the measure $\nu = P_{X}$. A crucial feature of our analysis involves establishing explicit learning rates that are \textit{adaptive} to this relationship between kernel complexity (Assumptions \ref{EVD} and \ref{EMB}) and the CME continuity (Assumption \ref{Hypothesis Space}). 

As mentioned previously, Assumptions \ref{EVD} and \ref{EMB} are borrowed directly from \cite{fischer2020sobolev}. However, our assumption \ref{Hypothesis Space} requiring only boundedness of $C^{\beta}_{Y|X}$ is significantly weaker than the Hilbert-Schmidt condition that would result from a direct generalization of the source condition in \cite{fischer2020sobolev} to operator-valued RKHSs. Indeed, the use of a more general source condition in Assumption \ref{Hypothesis Space} and operator subexponentiality in Assumption \ref{Subexponential} distinguishes our analysis from that of \cite{fischer2020sobolev} and requires the development of additional approximation machinery to obtain operator norm learning rates (see Appendix C and section \ref{Theorem 4 Proof})


\subsubsection{Example: Markov Operators}
\par To further demonstrate the generality of Assumption \ref{Hypothesis Space}, we consider an example involving Markov transition operators, which have recently found an elegant connection to CMEs (see e.g. \cite{mollenhauer2020nonparametric, mollenhauer2020kernel}). Although this example is presented to provide a concrete comparison with existing applications of conditional embeddings \citep{grunewalder2012modelling, lever2016compressed}, it should be noted that the argument applies to any setting where the input and output variables range over the same measure space (i.e. $\mathcal{X} = \mathcal{Y}$).

\par Let $\{X_t\}_{t \geq 0}$ be a Markov process on a state space $S \subset \mathbb{R}^d$. Fix $\tau > 0$, and let $p_{\tau}(y | x) = P(X_{t + \tau} = y | X_{t} = x)$ be the conditional transition density. Then, we may define the Koopman operator $P_{\tau}$ acting on an observable $\phi \in \mathcal{F}$ in some suitable function space $\mathcal{F}$ by:
\begin{equation*}
    (P_{\tau}\phi)(x) = \int_{S} p_{\tau}(y|x)\phi(y)dy
\end{equation*}

In applications, we are often interested in building empirical approximations to $P_{\tau}$ which typically require restricting the domain of $P_{\tau}$ to an amenable space. Recently, kernel methods have been proposed for this purpose \citep{klus2020eigendecompositions, mollenhauer2020kernel, klus2018data}, where $\mathcal{F}$ is set to some RKHS $\mathcal{H}_K$ typically over $L^2(\nu)$ (here $\nu$ is the stationary measure invariant under $P_{\tau}$). Indeed, \cite{klus2020eigendecompositions} demonstrate that when $\mathcal{F}$ is an RKHS, the Koopman operator $P_{\tau}$ is simply the dual of the conditional mean embedding mapping $P(X_t = \cdot)$ to $P(X_{t + \tau} = \cdot)$, thereby enabling the straightforward application of CME machinery towards the empirical estimation of $P_{\tau}$. However, their construction notably requires $P_{\tau}$ to be invariant over $\mathcal{H}_K$, which as mentioned in \cite{mollenhauer2020kernel} and \cite{das2020koopman} is quite restrictive, and equivalent to the assumption that the CME of $p_{\tau}$ exists from $\mathcal{H}_K$ to $\mathcal{H}_K$. Notably,  this assumption often introduces a model error by requiring the (possibly weak) approximation of $P_{\tau}$ in $\mathcal{L}(\mathcal{H}_K, \mathcal{H}_K)$. We observe that this assumption is significantly relaxed in our framework --- indeed by Lemma \ref{Characterization of Hypothesis Space}, we require that $P_{\tau}$ merely be a bounded operator on $\mathcal{H}_K$ with range in $\mathcal{H}^{\beta}_{K}$ for some $\beta \in (0, 1]$ (the latter space being strictly larger than $\mathcal{H}_K$ when $\beta < 1$). Hence, we may apply the new misspecified learning rates developed here towards the data-driven estimation of the Koopman operator $P_{\tau}$ in much broader settings. 

\section{TECHNICAL CONTRIBUTIONS} \label{Contributions}
We first present our main result in Theorem \ref{Main Result}. As expected, we achieve a faster learning rate as $\gamma \to 0$, i.e. as the norm $||\cdot ||_{\gamma}$ weakens. 


\begin{theorem}
\label{Main Result}
Suppose Assumptions 1-4 are satisfied, and that $\sup_{x \in \mathcal{X}} ||\mu_{Y|x}||_{L} \leq \tilde{C} < \infty$. Then, let $\lambda_n \asymp \Big(\frac{\log^r n}{n}\Big)^{\frac{1}{\max \{\alpha,  \beta + p\}}}$ for some $r > 1$. Then there exists a constant $K > 0$ (independent of $n$ and $\delta$), such that for $0 < \gamma < \beta$:
\begin{equation*}
    ||\hat{C}_{Y|X} - C_{Y|X}||_{\gamma} \leq K\log(\delta^{-1})\Big(\frac{n}{\log^r n}\Big)^{-\frac{\beta - \gamma}{2\max \{\alpha,  \beta + p\}}}
\end{equation*}
with probability $1 - 2\delta$.
\end{theorem}

The exponent $\frac{\beta - \gamma}{2\max \{\alpha,  \beta + p\}}$ illustrates that the learning rate hinges quite naturally on the comparison between $\alpha$ and $\beta$. Intuitively, the exponent $\alpha$ characterizes the boundedness of our kernel, while $\beta$ characterizes the boundedness of the conditional mean operator. The sizes of $\alpha$ and $\beta$ are related inversely to specification, with our problem being more strongly specified as $\alpha \to 0$ and $\beta \to 1$. We therefore expect to achieve faster learning rates for low $\alpha$ and high $\beta$. Indeed, when $\alpha > 2\beta$ (i.e. the kernel is relatively poorly bounded), then $\alpha = \max \{\alpha,  \beta + p\}$, which will limit the magnitude of the exponent and lead to a slow learning rate. Conversely, if $\beta > \alpha$, then we can bring our learning rate arbitrarily close to $\frac{\beta}{2(\beta + p)} \geq \frac{1}{4}$ (by Assumption \ref{Hypothesis Space}). Note, in this regime, the Sobolev norm learning rate $||\cdot||_{\gamma}$  is only useful for establishing uniform convergence rates when $\gamma \geq \alpha$. Indeed, here we can obtain a uniform error bound in $||\hat{\mu}_{Y|x} - \mu_{Y|x}||_{L}$ for the sample conditional mean embedding $\hat{\mu}_{Y|x}$ over all $x \in \mathcal{X}$. Moreover, as we will see later in Lemma \ref{orders}, the exponent $\alpha - \beta$ characterizes our ability to control the worst-case bias of our estimator $\sup_{x \in \mathcal{X}} ||\mu^{\lambda}_{Y|x} - \mu_{Y|x}||$ as $\lambda \to 0$, which likewise relates directly to the convergence of the sample embedding operator.

\begin{remark}
We note that the additional assumption $\sup_{x \in \mathcal{X}} ||\mu_{Y|x}||_{L} \leq \tilde{C}$ is not very restrictive, as this is easily satisfied when the kernel $\ell$ is bounded (recall we do not place an \textit{a priori} assumption on the boundedness of the output kernel $\ell$). 
\end{remark}
\begin{corollary}
\label{Main Corollary}
Suppose the hypotheses of Theorem \ref{Main Result}. Then, if $\beta > \alpha$, we obtain, with probability $1 - 2\delta$ and constant $K > 0$:
\begin{equation*}
    \sup_{x \in \mathcal{X}} ||\hat{\mu}_{Y|x} - \mu_{Y|x}||_{L} \leq K\log(\delta^{-1})\Big(\frac{n}{\log^r n}\Big)^{-\frac{\beta - \alpha}{2(\beta + p)}}
\end{equation*}
\end{corollary}
We emphasize that we are able to achieve learning rates roughly matching those in \cite{fischer2020sobolev} for scalar-valued regression, despite only requiring the continuity/boundedness of our target $C^{\beta}_{Y|X}$ rather than the stronger smoothness source condition imposed on the regression function in \cite{fischer2020sobolev}. Moreover, we note that we obtain a roughly similar $\frac{\log n}{n}$ base observed in \cite{grunewalder2012conditional} for finite-dimensional RKHSs, which is unsurprising as the latter is based off the regularized learning rates of 
\cite{caponnetto2007optimal}, which is foundational in the scalar-valued kernel regression literature.  

\subsection{Proof of Theorem \ref{Main Result}}
\label{Theorem 4 Proof}
To estimate the error $||\hat{C}_{Y|X}  - C_{Y | X}||_{\gamma} = ||\hat{C}_{YX}(\hat{C}_{XX} + \lambda I)^{-1}  - C_{Y | X}||_{\gamma}$, we follow the standard procedure by separating into bias and variance terms, and bounding each term independently. Namely, we write:
\begin{small}
\begin{equation}
\label{Bias-Variance Breakdown}
    ||\hat{C}_{Y|X}  - C_{Y | X}||_{\gamma} \leq ||\hat{C}_{Y|X}  - C^{\lambda}_{Y|X}||_{\gamma} + ||C^{\lambda}_{Y|X} - C_{Y | X}||_{\gamma} 
\end{equation}
\end{small}

Our main tool will be Theorem \ref{Embedding Concentration}, where we estimate the variance by notably applying the subexponential condition in Assumption \ref{Subexponential} and a new operator Bernstein inequality derived in Lemma C.3 in the Appendix, which may be of independent interest. Lemma C.3 is crucial in our analysis, as it enables us to quantify the variance in \eqref{Bias-Variance Breakdown} \textit{directly in operator norm}, rather than embedding the operators in a product RKHS, which implicitly requires them to be Hilbert-Schmidt (see discussion in e.g. \cite{park2020measure, mollenhauer2020nonparametric}). Like in \cite{fischer2020sobolev}, our variance bound in Theorem \ref{Embedding Concentration} is expressed implicitly in terms of the worst-case bias. Hence, we first discuss the estimation of this bias term $||C^{\lambda}_{Y|X} - C_{Y | X}||_{\gamma}$. 

\subsubsection{Bounding the Bias}
\label{Bounding the Bias}
\par In the following result, we seek to estimate several different measures of the bias between $\mu^{\lambda}_{Y|X}$ and $\mu_{Y|X}$, that relate to the various spectral properties of the covariance operators arising in Theorem \ref{Embedding Concentration}. We will see that while the ``average'' bias $\mathbb{E}_{X}[||\mu^{\lambda}_{Y|X} - \mu_{Y|X}||^2_{L}]$ can always be shown to decay polynomially at order $\beta - p$ as $\lambda \to 0$, estimating the worst-case bias is less straightforward. Notably, we can only demonstrate the polynomial decay of the latter quantity when $\beta > \alpha$, i.e. the ``nice'' regime when the conditional embedding can be expressed as a continuous operator acting on bounded RKHS, leading to uniform convergence rates in the output space $\mathcal{H}_L$. When $\beta \leq \alpha$, we can merely bound this worst-case bias in a way sufficient to achieve the learning rates in Theorem \ref{Main Result}. 

\par We argue that imposing a continuity constraint on $C^{\beta}_{Y|X}$ in Assumption \ref{Hypothesis Space} is more natural for studying uniform convergence of $\hat{\mu}_{Y|x}$, rather than the stronger Hilbert-Schmidt criteria often imposed in vector-valued regression. Indeed, estimating the bias $||\mu^{\lambda}_{Y|x} - \mu_{Y|x}||_{L}$ and sample error $||\hat{\mu}_{Y|x} - \mu_{Y|x}||_{L}$ involve quantifying distances in the output RKHS $\mathcal{H}_L$, the codomain of the true ($C^{\beta}_{Y|X}$), sample ($\hat{C}_{Y|X}$), and regularized ($C^{\lambda}_{Y|X}$) conditional embedding operators. Since $\mu_{Y|x}, \hat{\mu}_{Y|x}$, and $\mu^{\lambda}_{Y|x}$ lie more specifically in the range of their respective operators ($C^{\beta}_{Y|X}, \hat{C}_{Y|X}$, and $C^{\lambda}_{Y|X}$), intuitively, it is sufficient to constrain the operator norms of the latter to obtain uniform upper bounds in $\mathcal{H}_L$. However, we must note that additionally requiring $C^{\beta}_{Y|X}$ to be Hilbert-Schmidt would allow us to achieve the polynomial decay of the expected bias in \eqref{eq: Expected Deviation} for any $\beta \in (0, 2)$, without requiring $\beta > p$ as in Assumption \ref{Hypothesis Space} (elaborated in the remark following the proof of Lemma \ref{orders} in Appendix B). We view this tradeoff to be quite minor with respect to elimination of the Hilbert-Schmidt requirement on $C^{\beta}_{Y|X}$ (as discussed in section \ref{Assumption Discussion Section}). 
\begin{small}
\begin{lemma}
\label{orders}
Suppose Assumptions 1-4 and $\sup_{x \in \mathcal{X}} ||\mu_{Y|x}||_{L} \leq \tilde{C}$. Then, we have for all $0 < \gamma < \beta < 2$:
\begin{align}
   & \mathbb{E}_{X}[||\mu^{\lambda}_{Y|X} - \mu_{Y|X}||^2_{L}]   \leq DB\lambda^{\beta - p} \label{eq: Expected Deviation} \\
   & M^2(\lambda)  \equiv \sup_{x \in \mathcal{X}} ||\mu_{Y|x} - \mu^{\lambda}_{Y|x}||^2_{L} \leq (\tilde{C}^2 + ||k^{\alpha}||^2_{\infty}B^2)\lambda^{-(\alpha - \beta)_{+}} \label{eq: Max Deviation} \\
   & M_{\lambda} \equiv ||\mathbb{E}[(\mu_{Y|x} - \mu^{\lambda}_{Y|x}) \otimes (\mu_{Y|x} - \mu^{\lambda}_{Y|x})]|| \leq B^2\lambda^{\beta} \label{eq: Expected Deviation Norm}\\
   & ||C^{\lambda}_{Y|X} - C_{Y|X}||_{\gamma} \leq B\lambda^{\frac{\beta - \gamma}{2}} \label{eq: CME Deviation Norm}
\end{align}
\end{lemma}
\end{small}

\subsubsection{Bounding the Variance}
We now present our primary estimate, where we demonstrate that for sufficiently large $n$, the ``variance'' of the sample CME (in $||\cdot||_{\gamma}$) can be estimated implicitly via the bias. Specifically, we use a new operator Bernstein inequality (detailed in Appendix C) and the framework of Theorem 6.8 in \cite{fischer2020sobolev}, to demonstrate the concentration of  $\hat{C}_{Y|X}$ around $C_{Y|X}$ for a fixed $\lambda$. 
\begin{small}
\begin{theorem}
\label{Embedding Concentration}
Suppose Assumptions 1-4 hold. Let $\sigma^2 = \text{tr}(V)$ (where $V$ is defined in Assumption \ref{Subexponential}). Define:
\begin{align*}
    \mathcal{N}(\lambda) & = \text{tr}(C_{\nu}(C_{\nu} + \lambda)^{-1}) \\
    Q & = \max \{M(\lambda), R\} \\
    g_{\lambda} & = \log \Big(2e\mathcal{N}(\lambda)\frac{||C_{\nu}|| + \lambda}{||C_{\nu}||}\Big) \\
    \rho_{\lambda} & = \mathbb{E}\Big[(\mu^{\lambda}_{Y|X} - \mu_{Y|X}) \otimes (\mu^{\lambda}_{Y|X} - \mu_{Y|X})\Big] \\
    \eta & = \max\Big\{\frac{(\sigma^2 + M^2(\lambda))||C_{\nu}||}{||C_{\nu}|| + \lambda}, ||\mathcal{N}(\lambda)V + \frac{||k^{\alpha}||^2_{\infty}}{\lambda^{\alpha}}\rho_{\lambda}||\Big\} \\
    \beta(\delta) & = \log \Big(\frac{4((2\sigma^2 + M^2(\lambda))\mathcal{N}(\lambda) + \frac{||k^{\alpha}||^2_{\infty}}{\lambda^{\alpha}}\text{tr}(\rho_{\lambda}))}{\eta \delta}\Big)
\end{align*}
Then, for $n \geq 8||k^{\alpha}||^2_{\infty}\log(\delta^{-1}) g_{\lambda}\lambda^{-\alpha}$:
\begin{equation}
\label{Variance Bound}
    ||\hat{C}_{Y|X} - C^{\lambda}_{Y|X}||_{\gamma} \leq 3\lambda^{-\frac{\gamma}{2}}\Big(\frac{16Q||k^{\alpha}||_{\infty}\beta(\delta)}{\lambda^{\frac{\alpha}{2}}n} + 8\sqrt{\frac{\eta\beta(\delta)}{n}}\Big) 
\end{equation}
with probability $1 - 2\delta$
\end{theorem}
\end{small}
The proof of Theorem \ref{Main Result} then follows by substituting the bias estimates in Lemma \ref{orders} into \eqref{Variance Bound}, combining with the operator bias bound in \eqref{eq: CME Deviation Norm}, and considering the behavior of the resulting bound as $\lambda_n \asymp \Big(\frac{\log^r n}{n}\Big)^{\frac{1}{\max \{\alpha,  \beta + p\}}}$ as $n \to \infty$. A full proof of these three results can be found in Appendix B.

\section{DISCUSSION}
In this paper, we derive novel learning rates for conditional mean embeddings under a new misspecified framework that significantly relaxes the Hilbert-Schmidt criteria currently required to guarantee uniform convergence on infinite-dimensional RKHS. This relaxation reduces the need to explicitly verify the smoothness of the learning target, which can often be difficult or counterintuitive. Our results hopefully enable the much broader application of existing ML/RL algorithms for conditional mean embeddings to more complex, misspecified settings involving infinite dimensional RKHS and continuous state spaces. 

There are several remaining questions. Firstly, complementary lower bounds would be required for Theorem \ref{Main Result} to ensure the results presented here are indeed optimal. Given the ease in matching the upper bounds from the scalar-valued setting in \cite{fischer2020sobolev}, we suspect that our learning rates are likewise optimal in this setting, however verifying this would require further analysis. A further interesting question involves exploring how the framework developed here may generalize to other regularization approaches, such as spectral regularization, or quantile/expectile regression. 

\subsubsection*{Acknowledgements}
We would like to thank Yunzong Xu for insightful discussions and comments, and acknowledge the support of the MIT Data Science Lab. 

\bibliography{ref}

\begin{thebibliography}{24}
\providecommand{\natexlab}[1]{#1}
\providecommand{\url}[1]{\texttt{#1}}
\expandafter\ifx\csname urlstyle\endcsname\relax
  \providecommand{\doi}[1]{doi: #1}\else
  \providecommand{\doi}{doi: \begingroup \urlstyle{rm}\Url}\fi

\bibitem[Caponnetto and De~Vito(2007)]{caponnetto2007optimal}
Andrea Caponnetto and Ernesto De~Vito.
\newblock Optimal rates for the regularized least-squares algorithm.
\newblock \emph{Foundations of Computational Mathematics}, 7\penalty0
  (3):\penalty0 331--368, 2007.

\bibitem[Das and Giannakis(2020)]{das2020koopman}
Suddhasattwa Das and Dimitrios Giannakis.
\newblock Koopman spectra in reproducing kernel hilbert spaces.
\newblock \emph{Applied and Computational Harmonic Analysis}, 49\penalty0
  (2):\penalty0 573--607, 2020.

\bibitem[Fischer and Steinwart(2020)]{fischer2020sobolev}
Simon Fischer and Ingo Steinwart.
\newblock Sobolev norm learning rates for regularized least-squares algorithms.
\newblock \emph{J. Mach. Learn. Res.}, 21:\penalty0 205--1, 2020.

\bibitem[Fukumizu et~al.(2007)Fukumizu, Gretton, Sun, and
  Sch{\"o}lkopf]{fukumizu2007kernel}
Kenji Fukumizu, Arthur Gretton, Xiaohai Sun, and Bernhard Sch{\"o}lkopf.
\newblock Kernel measures of conditional dependence.
\newblock In \emph{NIPS}, volume~20, pages 489--496, 2007.

\bibitem[Fukumizu et~al.(2009)Fukumizu, Bach, Jordan,
  et~al.]{fukumizu2009kernel}
Kenji Fukumizu, Francis~R Bach, Michael~I Jordan, et~al.
\newblock Kernel dimension reduction in regression.
\newblock \emph{The Annals of Statistics}, 37\penalty0 (4):\penalty0
  1871--1905, 2009.

\bibitem[Gr{\"u}new{\"a}lder et~al.(2012{\natexlab{a}})Gr{\"u}new{\"a}lder,
  Lever, Baldassarre, Patterson, Gretton, and
  Pontil]{grunewalder2012conditional}
Steffen Gr{\"u}new{\"a}lder, Guy Lever, Luca Baldassarre, Sam Patterson, Arthur
  Gretton, and Massimilano Pontil.
\newblock Conditional mean embeddings as regressors.
\newblock In \emph{Proceedings of the 29th International Coference on
  International Conference on Machine Learning}, pages 1803--1810,
  2012{\natexlab{a}}.

\bibitem[Gr{\"u}new{\"a}lder et~al.(2012{\natexlab{b}})Gr{\"u}new{\"a}lder,
  Lever, Baldassarre, Pontil, and Gretton]{grunewalder2012modelling}
Steffen Gr{\"u}new{\"a}lder, Guy Lever, Luca Baldassarre, Massimilano Pontil,
  and Arthur Gretton.
\newblock Modelling transition dynamics in mdps with rkhs embeddings.
\newblock In \emph{Proceedings of the 29th International Coference on
  International Conference on Machine Learning}, pages 1603--1610,
  2012{\natexlab{b}}.

\bibitem[Klebanov et~al.(2020)Klebanov, Schuster, and
  Sullivan]{klebanov2020rigorous}
Ilja Klebanov, Ingmar Schuster, and TJ~Sullivan.
\newblock A rigorous theory of conditional mean embeddings.
\newblock \emph{SIAM Journal on Mathematics of Data Science}, 2\penalty0
  (3):\penalty0 583--606, 2020.

\bibitem[Klus et~al.(2018)Klus, N{\"u}ske, Koltai, Wu, Kevrekidis, Sch{\"u}tte,
  and No{\'e}]{klus2018data}
Stefan Klus, Feliks N{\"u}ske, P{\'e}ter Koltai, Hao Wu, Ioannis Kevrekidis,
  Christof Sch{\"u}tte, and Frank No{\'e}.
\newblock Data-driven model reduction and transfer operator approximation.
\newblock \emph{Journal of Nonlinear Science}, 28\penalty0 (3):\penalty0
  985--1010, 2018.

\bibitem[Klus et~al.(2020)Klus, Schuster, and
  Muandet]{klus2020eigendecompositions}
Stefan Klus, Ingmar Schuster, and Krikamol Muandet.
\newblock Eigendecompositions of transfer operators in reproducing kernel
  hilbert spaces.
\newblock \emph{Journal of Nonlinear Science}, 30\penalty0 (1):\penalty0
  283--315, 2020.

\bibitem[Lever et~al.(2016)Lever, Shawe-Taylor, Stafford, and
  Szepesv{\'a}ri]{lever2016compressed}
Guy Lever, John Shawe-Taylor, Ronnie Stafford, and Csaba Szepesv{\'a}ri.
\newblock Compressed conditional mean embeddings for model-based reinforcement
  learning.
\newblock In \emph{Proceedings of the AAAI Conference on Artificial
  Intelligence}, volume~30, 2016.

\bibitem[Lin and Cevher(2020)]{lin2020optimal}
Junhong Lin and Volkan Cevher.
\newblock Optimal convergence for distributed learning with stochastic gradient
  methods and spectral algorithms.
\newblock \emph{J. Mach. Learn. Res.}, 21:\penalty0 147--1, 2020.

\bibitem[Lin et~al.(2020)Lin, Rudi, Rosasco, and Cevher]{lin12020optimal}
Junhong Lin, Alessandro Rudi, Lorenzo Rosasco, and Volkan Cevher.
\newblock Optimal rates for spectral algorithms with least-squares regression
  over hilbert spaces.
\newblock \emph{Applied and Computational Harmonic Analysis}, 48\penalty0
  (3):\penalty0 868--890, 2020.

\bibitem[Minh(2010)]{minh2010some}
Ha~Quang Minh.
\newblock Some properties of gaussian reproducing kernel hilbert spaces and
  their implications for function approximation and learning theory.
\newblock \emph{Constructive Approximation}, 32\penalty0 (2):\penalty0
  307--338, 2010.

\bibitem[Mollenhauer and Koltai(2020)]{mollenhauer2020nonparametric}
Mattes Mollenhauer and P{\'e}ter Koltai.
\newblock Nonparametric approximation of conditional expectation operators.
\newblock \emph{arXiv preprint arXiv:2012.12917}, 2020.

\bibitem[Mollenhauer et~al.(2020)Mollenhauer, Klus, Sch{\"u}tte, and
  Koltai]{mollenhauer2020kernel}
Mattes Mollenhauer, Stefan Klus, Christof Sch{\"u}tte, and P{\'e}ter Koltai.
\newblock Kernel autocovariance operators of stationary processes: Estimation
  and convergence.
\newblock \emph{arXiv preprint arXiv:2004.00891}, 2020.

\bibitem[Park and Muandet(2020)]{park2020measure}
Junhyung Park and Krikamol Muandet.
\newblock A measure-theoretic approach to kernel conditional mean embeddings.
\newblock \emph{Advances in Neural Information Processing Systems}, 33, 2020.

\bibitem[Song et~al.(2009)Song, Huang, Smola, and Fukumizu]{song2009hilbert}
Le~Song, Jonathan Huang, Alex Smola, and Kenji Fukumizu.
\newblock Hilbert space embeddings of conditional distributions with
  applications to dynamical systems.
\newblock In \emph{Proceedings of the 26th Annual International Conference on
  Machine Learning}, pages 961--968, 2009.

\bibitem[Song et~al.(2010)Song, Gretton, and Guestrin]{song2010nonparametric}
Le~Song, Arthur Gretton, and Carlos Guestrin.
\newblock Nonparametric tree graphical models.
\newblock In \emph{Proceedings of the Thirteenth International Conference on
  Artificial Intelligence and Statistics}, pages 765--772, 2010.

\bibitem[Steinwart(2019)]{steinwart2019convergence}
Ingo Steinwart.
\newblock Convergence types and rates in generic karhunen-loeve expansions with
  applications to sample path properties.
\newblock \emph{Potential Analysis}, 51\penalty0 (3):\penalty0 361--395, 2019.

\bibitem[Steinwart and Christmann(2008)]{steinwart2008support}
Ingo Steinwart and Andreas Christmann.
\newblock \emph{Support vector machines}.
\newblock Springer Science \& Business Media, 2008.

\bibitem[Steinwart and Scovel(2012)]{steinwart2012mercer}
Ingo Steinwart and Clint Scovel.
\newblock Mercer’s theorem on general domains: On the interaction between
  measures, kernels, and rkhss.
\newblock \emph{Constructive Approximation}, 35\penalty0 (3):\penalty0
  363--417, 2012.

\bibitem[Tropp(2015)]{tropp2015introduction}
Joel~A Tropp.
\newblock An introduction to matrix concentration inequalities.
\newblock \emph{Foundations and Trends{\textregistered} in Machine Learning},
  8\penalty0 (1-2):\penalty0 1--230, 2015.

\bibitem[Zhou(2002)]{zhou2002covering}
Ding-Xuan Zhou.
\newblock The covering number in learning theory.
\newblock \emph{Journal of Complexity}, 18\penalty0 (3):\penalty0 739--767,
  2002.

\end{thebibliography}





\clearpage
\appendix

\thispagestyle{empty}

\onecolumn \makesupplementtitle
\section{Proofs for Sections \ref{CME Interpolation} and \ref{Assumptions Section}}
\label{Assumptions Proofs}
\begin{proof}[Proof of Lemma \ref{Equivalence of Mean Embeddings}]
We must demonstrate that $C_{Y|X} \circ I^{*}_{1, \beta, \nu}: \mathcal{H}^{\beta}_K \to \mathcal{H}_L$ satisfies the definition of the conditional mean embedding in Definition \ref{CME Def} where the input space is taken as $\mathcal{H}_{K}^{\beta}$ (instead of $\mathcal{H}_K$). Thus, we must show that $C_{Y|X} \circ I^{*}_{1, \beta, \nu} k^{\beta}(x, \cdot) = \mu_{Y|x}$. We first observe that, for any $f \in \mathcal{H}_{K}$ and $x \in \mathcal{X}$, we have:
\begin{align*}
    \langle I_{1, \beta, \nu}f, k^{\beta}(x, \cdot) \rangle_{\mathcal{H}^{\beta}_{K}} & = \langle f, k^{\beta}(x, \cdot) \rangle_{\mathcal{H}^{\beta}_{K}}  \\
    & = f(x) \\
    & = \langle f, k(x, \cdot) \rangle_{K} 
\end{align*}
Hence, we have that $I^{*}_{1, \beta, \nu}k^{\beta}(x, \cdot) = k(x, \cdot)$. Therefore, by the definition of $C_{Y|X}$ in Definition \ref{CME Def}, we have:
\begin{align*}
   (C_{Y|X} \circ I^{*}_{1, \beta, \nu}) k^{\beta}(x, \cdot) &= C_{Y|X} k(x, \cdot) \\
   & = \mu_{Y|x}
\end{align*}
and we obtain our result. 
\end{proof}

\begin{proof}[Proof of Lemma \ref{Operator Norms in Interpolation Spaces}]
Since $\{\mu^{\frac{\beta}{2}}_i e_i\}_{i = 1}^{\infty}$ is an orthonormal basis for $\mathcal{H}^{\beta}_K$, we may express any $f \in \mathcal{H}^{\beta}_{K}$ as $f = \sum_{i = 1}^{\infty} \langle f, \mu^{\frac{\beta}{2}}_i e_i \rangle_{\mathcal{H}^{\beta}_K}\mu^{\frac{\beta}{2}}_i e_i$. Hence, we have:
\begin{align*}
    \langle f, C_{\beta, \gamma, \nu} (\mu^{\frac{\beta}{2}}_i e_i) \rangle_{\mathcal{H}^{\beta}_K} & = \langle f, I^{*}_{\beta, \gamma, \nu}I_{\beta, \gamma, \nu} (\mu^{\frac{\beta}{2}}_i e_i) \rangle_{\mathcal{H}^{\beta}_K} \\
    & = \langle I_{\beta, \gamma, \nu}f, I_{\beta, \gamma, \nu} (\mu^{\frac{\beta}{2}}_i e_i) \rangle_{\mathcal{H}^{\gamma}_K} \\
    & = \langle f, \mu^{\frac{\beta}{2}}_i e_i \rangle_{\mathcal{H}^{\gamma}_K} \\
    & = \Big \langle \sum_{i = 1}^{\infty} \langle f, \mu^{\frac{\beta}{2}}_i e_i \rangle_{\mathcal{H}^{\beta}_K}\mu^{\frac{\beta}{2}}_i e_i, \mu^{\frac{\beta}{2}}_i e_i \Big\rangle_{\mathcal{H}^{\gamma}_K} \\
    & = \mu^{\beta - \gamma}_i \langle f, \mu^{\frac{\beta}{2}}_i e_i \rangle_{\mathcal{H}^{\beta}_K}
\end{align*}
where the final step follows from the fact that $\{\mu_i^{\frac{\gamma}{2}}e_i\}_{i = 1}^{\infty}$ is an orthonormal basis in $\mathcal{H}^{\gamma}_K$. Hence $C_{\beta, \gamma, \nu}$ is a positive self-adjoint operator on $\mathcal{H}^{\beta}_K$ with eigenvalues $\{\mu^{\beta - \gamma}_i\}_{i = 1}^{\infty}$ and an orthonormal basis of eigenfunctions $\{\mu^{\frac{\beta}{2}}_i e_i\}_{i = 1}^{\infty}$. Moreover, since $C_{\beta, \gamma, \nu} = I^{*}_{\beta, \gamma, \nu}I_{\beta, \gamma, \nu}$ by definition and $I_{\beta, \gamma, \nu}$ is the canonical embedding of $H^{\beta}_{K}$ into $H^{\gamma}_{K}$, it follows that the action of $I^{*}_{\beta, \gamma, \nu}: H^{\gamma}_{K} \to  \mathcal{H}^{\beta}_{K}$ can be characterized as:
\begin{equation}
\label{Action of Adjoint}
    I^{*}_{\beta, \gamma, \nu} e_i = \mu^{\beta - \gamma}_{i} e_i \hspace{4mm} \nu-\text{almost surely}
\end{equation}
for all $i \in \mathbb{N}$. Now, let $B_{\ell^2}$ denote the unit ball in $\ell^2$. Then, we have that for any linear operator $T: \mathcal{H}^{\beta}_K \to \mathcal{H}_L$:
\begin{align*}
    ||T||_{\beta, \gamma} & = ||T \circ I^{*}_{\beta, \gamma, \nu}|| \\
    & = \sup_{a \in B_{\ell^2}} \Big\|\sum_{i} a_i \mu^{\frac{\gamma}{2}}_i (T \circ I^{*}_{\beta, \gamma, \nu})  e_i \Big\| \\
    & = \sup_{a \in B_{\ell^2}} \Big\|\sum_{i} \mu^{\beta - \frac{\gamma}{2}}_i a_i (Te_i)\Big\| \hspace{10mm} \text{by \eqref{Action of Adjoint}} \\
    & = \sup_{a \in B_{\ell^2}} \Big\|\sum_{i} a_i\mu^{\frac{\beta - \gamma}{2}}_i  T(\mu^{\frac{\beta}{2}}_ie_i)\Big\| \\
    & = \sup_{a \in B_{\ell^2}} \Big\|\sum_{i} a_i (T \circ C_{\beta, \gamma, \nu}^{\frac{1}{2}})\mu^{\frac{\beta}{2}}_i e_i\Big\| \\
    & = ||T \circ C_{\beta, \gamma, \nu}^{\frac{1}{2}}||
\end{align*}
where again the last and second equalities follow from the fact that $\{\mu^{\frac{\beta}{2}}_i e_i\}_{i \in \mathbb{N}}$ and $\{\mu^{\frac{\gamma}{2}}_i e_i\}_{i \in \mathbb{N}}$ are orthonormal bases in $\mathcal{H}^{\beta}_K$ and $\mathcal{H}^{\gamma}_K$, respectively. 
\end{proof}
The following result demonstrates that if $\mathbb{E}_Y[\ell(Y, Y)] < \infty$, then $C^{\beta}_{Y|X}$ is always bounded when it exists.
\begin{lemma}
\label{Existence Implies Bounded}
Suppose that $C^{\beta}_{Y|X}$ exists and $\mathbb{E}_Y[\ell(Y, Y)] < \infty$. Then, $C^{\beta}_{Y|X}$ is bounded. 
\end{lemma}
\begin{proof}
Define $C_{\beta, XY} \equiv \mathbb{E}_{XY}[k^{\beta}(X, \cdot) \otimes l(Y, \cdot)]$. Note, that this operator is analogous to the cross-covariance operator $C_{XY}$ defined in section \ref{Preliminaries}, except that the feature vectors $k(x, \cdot)$ have been replaced by $k^{\beta}(x, \cdot)$, since $C_{\beta, XY}$ maps between the RKHS $\mathcal{H}_L$ and $\mathcal{H}^{\beta}_K$. Similarly, it is easy to see that the covariance operator of $X$ over $\mathcal{H}^{\beta}_{K}$ is simply $\mathbb{E}_{X}[k^{\beta}(X, \cdot) \otimes k^{\beta}(X, \cdot)]$ and equivalent to $C_{\beta, 0, \nu} = I^{*}_{\beta, 0, \nu}I_{\beta, 0, \nu}$ (as defined in Lemma \ref{Operator Norms in Interpolation Spaces} and Definition \ref{Sobolev Norm}; note here $I_{\beta, 0, \nu}$ is simply the embedding of $\mathcal{H}^{\beta}_K$ in $L^2(\nu)$). Thus, by the discussion in section \ref{Preliminaries} and \cite{klebanov2020rigorous}, it follows that when $C^{\beta}_{Y|X}$ exists it is given by $(C_{\beta, 0, \nu}^{\dagger}C_{\beta, XY})^{*}$. Hence, in order to demonstrate that $C^{\beta}_{Y|X}$ is bounded, we must only demonstrate that $C_{\beta, 0, \nu}$ and $C_{\beta, XY}$ are bounded and then apply Theorem A.1 in \cite{klebanov2020rigorous}. It is clear that $C_{\beta, 0, \nu}$ shares the same eigenvalues as $T^{\beta}_{\nu}$ (just as $T_{\nu}$ and $C_{\nu}$), and hence $||C_{\beta, 0, \nu}|| = \mu^{\beta}_1 < \infty$. To see that $C_{\beta, XY}: \mathcal{H}_L \to \mathcal{H}^{\beta}_K$ is bounded, we note that, for $f \in \mathcal{H}_L$ with $||f||_{L} \leq 1$, we have:
\begin{align}
    ||C_{\beta, XY} f||_{\mathcal{H}^{\beta}_K} & = \Big\|\mathbb{E}_{XY}[k^{\beta}(X, \cdot) \langle l(Y, \cdot), f \rangle_{L}]\Big\|_{\mathcal{H}^{\beta}_K} \nonumber \\
    & \leq ||f||_{L}\Big\|\mathbb{E}_{XY}[k^{\beta}(X, \cdot)\sqrt{l(Y, Y)}]\Big\|_{\mathcal{H}^{\beta}_K} \label{eq: CS1} \\
    & \leq \mathbb{E}_{XY}[\|k^{\beta}(X, \cdot)\|\sqrt{l(Y, Y)}] \label{eq: Jensen's} \\
    & = \mathbb{E}_{XY}[\sqrt{k^{\beta}(X, X)l(Y, Y)}] \nonumber \\
    & \leq \sqrt{\mathbb{E}_X[k^{\beta}(X, X)]\mathbb{E}_Y[l(Y, Y)]} \label{eq: CS2} \\
    & < \infty \label{eq: Finiteness}
\end{align}
where \eqref{eq: CS1} follows from Cauchy-Schwarz, \eqref{eq: Jensen's} follows from Jensen's inequality and the fact that $||f||_{L} \leq 1$, \eqref{eq: CS2} follows from Cauchy-Schwarz, and finally \eqref{eq: Finiteness} follows from the assumption $\mathbb{E}_Y[\ell(Y, Y)] < \infty$ and the fact that $\mathbb{E}_{X}[k^{\beta}(X, X)] = \mathbb{E}_{X}[\sum_{i = 1}^{\infty} \mu^{\beta}_i e^2_i(X)] < \infty$, since $\mathcal{H}^{\beta}_{K}$ is implicitly an RKHS (since the CME $C^{\beta}_{Y|X}$ is well-defined) and hence satisfies \eqref{Interpolation is RKHS}. Hence $C_{\beta, XY}$ is bounded, and our result follows from Theorem A.1 in \cite{klebanov2020rigorous}. Moreover, if Assumption \ref{EVD} is satisfied for some $p > 1$, it follows that since $\infty > \mathbb{E}_{X}[k^{\beta}(X, X)] = \mathbb{E}_{X}[\sum_{i = 1}^{\infty} \mu^{\beta}_i e^2_i(X)] = \sum_{i = 1}^{\infty} \mu^{\beta}_i \geq c\sum_{i = 1}^{\infty} i^{-p^{-1}\beta}$, that $\beta > p$. 
\end{proof}
\begin{proof}[Proof of Lemma \ref{Characterization of Hypothesis Space}]
Let $g_f(\cdot) = \mathbb{E}[f(Y)|X = \cdot]$. We observe that for every $x \in \mathcal{X}$: 
\begin{align}
    g_f(x) & = \mathbb{E}_{Y|x}[f(Y)] \\
    & = \langle f, \mu_{Y|x} \rangle_{L} \\
    & = \langle f, C^{\beta}_{Y|X} k^{\beta}(x, \cdot) \rangle_{L} \\
    & = \langle (C^{\beta}_{Y|X})^{*} f, k^{\beta}(x, \cdot) \rangle_{\mathcal{H}^{\beta}_K} 
\end{align}
Since $g_f \in \mathcal{H}^{\beta}_K$ by assumption (recall this is implicit in the existence of $C^{\beta}_{Y|X}: \mathcal{H}^{\beta}_K \to \mathcal{H}_L$), we have that $g_f = (C^{\beta}_{Y|X})^{*} f$. The result then follows from:
\begin{align*}
    ||C^{\beta}_{Y|X}||^2 &= \sup_{||f||_L \leq 1} \sum_{i = 1}^{\infty} \langle f, C^{\beta}_{Y|X} \mu_i^{\frac{\beta}{2}} e_i \rangle_{L}^2  \\
    & = \sup_{||f||_L \leq 1} \sum_{i = 1}^{\infty} \langle (C^{\beta}_{Y|X})^{*}f,  \mu_i^{\frac{\beta}{2}} e_i \rangle_{\mathcal{H}^{\beta}_K}^2 \\
    & = \sup_{||f||_L \leq 1} \sum_{i = 1}^{\infty} \langle g_f, \mu_i^{\frac{\beta}{2}} e_i \rangle_{\mathcal{H}^{\beta}_K}^2 \\
    & = \sup_{||f||_L \leq 1} ||g_f||^2_{\mathcal{H}^\beta_K}
\end{align*}
The second part of the lemma follows directly from Lemma \ref{Existence Implies Bounded}. 
\end{proof}
\begin{proof}[Proof of Lemma \ref{Satisfying Moment Condition}]
We first note that, here $\pi$ may be any measure, and we only require that the compact imbedding $\mathcal{H}_{L} \hookrightarrow L^2(\pi)$ be injective (which ensures that $\{\eta^{\frac{1}{2}}_i f_i\}_{i = 1}^{\infty}$ is indeed an orthonormal basis for $\mathcal{H}_L$ by Theorem 3.3 in \cite{steinwart2012mercer}) Let $g_f(x) = \mathbb{E}_{Y|x}[f(Y)]$, for $f \in \mathcal{H}_L$. Then, we have that:

\begin{align}
    \mathbb{E}_{Y|x}\Big[\Big((l(Y, \cdot) - \mu_{Y|x}) \otimes (l(Y, \cdot) - \mu_{Y|x})\Big)^p\Big] & = \mathbb{E}_{Y|x}[||l(Y, \cdot) - \mu_{Y|x}||^{2p - 2} (l(Y, \cdot) - \mu_{Y|x}) \otimes (l(Y, \cdot) - \mu_{Y|x})] \nonumber \\
    & \preccurlyeq (2\ell)^{2p-2}\mathbb{E}_{Y|x}[(l(Y, \cdot) - \mu_{Y|x}) \otimes (l(Y, \cdot) - \mu_{Y|x})] \label{eq: Bounded Kernel} \\
    & \preccurlyeq (2\ell)^{2p-2}\mathbb{E}_{Y|x}[l(Y, \cdot) \otimes l(Y, \cdot)] \label{eq: Second Operator Moment}
\end{align}
where \eqref{eq: Bounded Kernel} follows from the fact that $\mu_{Y|x} = \mathbb{E}_{Y|x}[l(Y, \cdot)]$ by definition and $||l(y, \cdot)|| = \sqrt{l(y, y)} \leq l$ by assumption. Now, since:
\begin{equation*}
    l(y, \cdot) = \sum_{i = 1}^{\infty} \eta_i f_i(y) f_i
\end{equation*}
converges pointwise (Theorem 3.3 in \cite{steinwart2012mercer}), we have that for any $h \in \mathcal{H}_L$, 
\begin{align}
    \langle h, l(y, \cdot)\rangle_L^2 & = \Big \langle h, \sum_{i = 1}^{\infty} \eta_i f_i(y) f_i\Big \rangle_L^2 \nonumber \\
    & \leq \Big(\sum_{i = 1}^{\infty} \eta^{\gamma}_i f^2_i(y)\Big)\Big(\sum_{i = 1}^{\infty} \eta_i^{1 - \gamma} \langle h, \eta^{\frac{1}{2}}_i f_i \rangle_L^2\Big) \nonumber \\
    & \leq K\langle h, C^{1-\gamma}_{\pi} h \rangle_L \label{eq:Apply Uniform Bound} 
\end{align}
where \eqref{eq:Apply Uniform Bound} follows from the fact that $K \equiv \sum_{i = 1}^{\infty} \eta^{\gamma}_i f^2_i(y) < \infty$ by assumption, and in \eqref{eq:Apply Uniform Bound}, $C_{\pi}$ is defined analogously to $C_{\nu}$ in section \ref{Preliminaries}. Hence, for all $y \in \mathcal{Y}$, $l(y, \cdot) \otimes l(y, \cdot) \preccurlyeq KC^{1-\gamma}_{\pi}$ and: $$\mathbb{E}_{Y|x}\Big[\Big((l(Y, \cdot) - \mu_{Y|x}) \otimes (l(Y, \cdot) - \mu_{Y|x})\Big)^p\Big] \preccurlyeq (2\ell)^{2p-2}\mathbb{E}_{Y|x}[l(Y, \cdot) \otimes l(Y, \cdot)] \preccurlyeq K(2\ell)^{2p-2}C^{1-\gamma}_{\pi}$$ 
Finally, $\text{tr}\Big(C^{1 - \gamma}_{\pi}\Big) = \sum_{i} \eta^{1 - \gamma}_i \asymp   \sum_{i} i^{-q^{-1}(1 - \gamma)} < \infty$ since $\gamma < 1 - q$.  Hence, we obtain our result with $V = KC^{1 - \gamma}_{\pi}$ and $R = 2\ell$.
\end{proof}


\begin{remark}[Assumptions in Lemma \ref{Satisfying Moment Condition}]
A particularly illustrative case of the assumption $\eta_i = \mathcal{O}\Big(i^{-q^{-1}}\Big)$ occurs when the $\eta_i$ decay exponentially (such as when $l$ is the Gaussian kernel and $\pi$ is the Lebesgue measure), in which case it is easy to see that the decay condition holds for any $q \in (0, 1)$. Moreover, we note that our boundedness condition $\Big\|\sum_{i \in \mathbb{N}} \eta^{\gamma}_i f^2_i\Big\|_{L^{\infty}(\mathcal{Y})} < \infty$ is significantly weaker than requiring the uniform boundedness of the eigenfunctions ($\sup_{i \in \mathbb{N}} ||f_i||_{L^{\infty}(\mathcal{Y})} < \infty$), the latter of which is often violated even for $C^{\infty}$ kernels (see discussion in \cite{steinwart2012mercer} and \cite{zhou2002covering}). In fact, for the kernel in Example 1 of \cite{zhou2002covering}, it can be shown that the requirement $\Big\|\sum_{i \in \mathbb{N}} \eta^{\gamma}_i f^2_i\Big\|_{L^{\infty}(\mathcal{Y})} < \infty$, is satisfied for any choice of $\gamma \in \Big(\frac{\ln 8}{\ln 16}, 1\Big)$, despite $||f_i||_{L^{\infty}(\mathcal{Y})}$ growing exponentially. \textit{Most importantly, Lemma \ref{Satisfying Moment Condition} demonstrates that we can replace the requirement on the conditional distribution $Y|X$ in Assumption \ref{Subexponential} with a condition on $\mathcal{H}_L$ and thereby eliminate any constraints on $P(Y|X)$ in our hypotheses.}
\end{remark}
\section{Proof of Theorem \ref{Main Result}}
\begin{proof}[Proof of Lemma \ref{orders}]
We first note that:
\begin{align}
    \mu^{\lambda}_{Y|X} &= C_{YX}(C_{XX} + \lambda)^{-1}k(x, \cdot) \nonumber \\
    & = \mathbb{E}_{YX}[l(y, \cdot) \otimes k(x, \cdot)](C_{XX} + \lambda)^{-1}k(x, \cdot) \nonumber \\
    & = C^{\beta}_{Y|X}\mathbb{E}_{X}[k^{\beta}(x, \cdot) \otimes k(x, \cdot)](C_{XX} + \lambda)^{-1}k(x, \cdot) \label{eq: Cross-Covariance to Covariance}
\end{align}
where \eqref{eq: Cross-Covariance to Covariance} follows from the fact that $\mu_{Y|x} = \mathbb{E}_{Y|X = x}[l(Y, \cdot)] = C^{\beta}_{Y|X}k^{\beta}(x, \cdot)$ by the definition of the conditional embedding $C^{\beta}_{Y|X}$ on $\mathcal{H}^{\beta}_{K}$. We then observe that:
\begin{align*}
    \mathbb{E}_{X}[k^{\beta}(x, \cdot) \otimes k(x, \cdot)](C_{XX} + \lambda)^{-1}k(x, \cdot)  & = \mathbb{E}_{X}\Big[\Big(\sum_{i = 1}^{\infty} \mu^{\beta}_i e_i(X) e_i \Big) \otimes \Big(\sum_{i = 1}^{\infty} \mu_i e_i(X) e_i\Big)\Big] (C_{XX} + \lambda)^{-1}k(x, \cdot)\\
    & = \Big(\sum_{i = 1}^{\infty} \mu_i^{1 + \beta} e_i \otimes e_i \Big) (C_{XX} + \lambda)^{-1}k(x, \cdot) \\
    & = \Big(\sum_{i = 1}^{\infty} \frac{\mu_i^{1 + \beta}}{\mu_i + \lambda} e_i \otimes e_i\Big) k(x, \cdot) \\
    & = \sum_{i = 1}^{\infty} \frac{\mu_i^{1 + \beta}}{\mu_i + \lambda} e_i(x) e_i
\end{align*}
We thus have that:
\begin{align*}
    \mu^{\lambda}_{Y|X} - \mu_{Y|X} &= C^{\beta}_{Y|X}\Big(\mathbb{E}_{X}[k^{\beta}(x, \cdot) \otimes k(x, \cdot)](C_{XX} + \lambda)^{-1}k(x, \cdot)\Big) - C^{\beta}_{Y|X}k^{\beta}(x, \cdot) \\
    & = C^{\beta}_{Y|X}\Big(\sum_{i = 1}^{\infty} \frac{\mu_i^{1 + \beta}}{\mu_i + \lambda} e_i(x) e_i - \sum_{i = 0}^{\infty} \mu_i^{\beta} e_i(x) e_i\Big) \\
    & = \sum_{i = 1}^{\infty} \frac{\lambda}{\mu_i + \lambda} \cdot C^{\beta}_{Y|X} \mu_i^{\beta}e_i(x) e_i
\end{align*}
Thus, we can write:
\begin{align*}
    \mathbb{E}_{X}[||\mu_{Y|X} - \mu^{\lambda}_{Y|X}||^2_{L}] & = \mathbb{E}_{X}\Big[\Big|\Big|\sum_{i = 1}^{\infty} \frac{\lambda}{\lambda + \mu_i} C^{\beta}_{Y|X} \mu^{\beta}_i e_i(X)e_i\Big|\Big|^2_{L}\Big] \\
    & = \mathbb{E}_{X}\Big[\Big|\Big|\sum_{i = 1}^{\infty} \frac{\lambda \cdot \mu^{\frac{\beta}{2}}_i}{\lambda + \mu_i} C^{\beta}_{Y|X} \mu^{\frac{\beta}{2}}_i e_i(X)e_i\Big|\Big|^2_{L}\Big] \\
    & \leq \lambda^2 ||C^{\beta}_{Y|X}||^2 \mathbb{E}_{X}\Big[\sum_{i = 1}^{\infty} \Big(\frac{\mu^{\frac{\beta}{2}}_i}{\lambda + \mu_i}\Big)^2 e^2_i(X)\Big] \\
    & = \lambda^{2} ||C_{Y|X}^{\beta}||^2 \sum_{i = 1}^{\infty} \Big(\frac{\mu^{\frac{\beta}{2}}_i}{\lambda + \mu_i}\Big)^2 \\
    & \leq D\lambda^{\beta - p} ||C_{Y|X}^{\beta}||^2
\end{align*}
where the last line follows from Lemma \ref{Convergent Series}. Moreover, we have, for any $x \in X$:
\begin{align}
    ||\mu_{Y|x} - \mu^{\lambda}_{Y|x}||^2_{L} &= \Big|\Big|\sum_{i = 1}^{\infty} \frac{\lambda}{\lambda + \mu_i} C^{\beta}_{Y|X} \mu^{\beta}_i e_i(x)e_i\Big|\Big|^2_{L} \nonumber \\
    & = \Big|\Big|\sum_{i = 1}^{\infty} \frac{\lambda \cdot \mu_i^{\frac{\beta - \alpha}{2}}}{\lambda + \mu_i} \cdot \mu^{\frac{\alpha}{2}}e_i(x) \cdot C^{\beta}_{Y|X} \mu^{\frac{\beta}{2}}_i e_i\Big|\Big|^2_{L} \nonumber \\
    & \leq \Big(\sum_i \Big(\frac{\lambda \cdot \mu_i^{\frac{\beta - \alpha}{2}}}{\lambda + \mu_i}\Big)^2 \mu^{\alpha}e^2_i(x)\Big) ||C^{\beta}_{Y|X}||^2 \label{eq: Op Norm Def} \\
    & \leq \Big(\sup_{i} \Big(\frac{\lambda \cdot \mu_i^{\frac{\beta - \alpha}{2}}}{\lambda + \mu_i}\Big)^2\Big) \cdot \sum_i \mu^{\alpha} e^2_i(x) \cdot ||C^{\beta}_{Y|X}||^2 \nonumber \\
    & \leq \lambda^{\beta - \alpha} ||k^{\alpha}||^2_{\infty}||C^{\beta}_{Y|X}||^2 \nonumber
\end{align}
when $\beta > \alpha$ (here \eqref{eq: Op Norm Def} follows from the fact that $\{\mu_i^{\frac{\beta}{2}}e_i\}_{i = 1}^{\infty}$ is an orthonormal basis for $\mathcal{H}_{K}^{\beta}$ and the last line follows from Lemma A.1 in \cite{fischer2020sobolev}). When $\beta < \alpha$, we have that:
\begin{align*}
    ||\mu^{\lambda}_{Y|x}||^2_{L} & = \Big|\Big|\sum_{i = 1}^{\infty} \frac{\mu_i}{\mu_i + \lambda} \cdot C^{\beta}_{Y|X} \mu^{\beta}_i e_i(x) e_i\Big|\Big|_{L}^2 \\
        & = \Big|\Big|\sum_{i = 1}^{\infty} \frac{\mu_i^{1 + \frac{\beta - \alpha}{2}}}{\lambda + \mu_i} \cdot \mu_i^{\frac{\alpha}{2}}e_i(x) \cdot C^{\beta}_{Y|X} \mu^{\frac{\beta}{2}}_i e_i\Big|\Big|^2_{L} \\
    & = \Big(\sum_i \Big(\frac{\mu_i^{1 + \frac{\beta - \alpha}{2}}}{\lambda + \mu_i}\Big)^2 \mu_i^{\alpha}e^2_i(x)\Big)||C^{\beta}_{Y|X}||^2 \\
    & \leq \lambda^{\beta - \alpha}||k^{\alpha}||^2_{\infty}||C^{\beta}_{Y|X}||^2
\end{align*}
where again the last line follows from Lemma A.1 in \cite{fischer2020sobolev}. Thus, we have for all cases:
\begin{align*}
    ||\mu_{Y|x} - \mu^{\lambda}_{Y|x}||_{L} & \leq ||\mu_{Y|x}||_{L} + ||\mu^{\lambda}_{Y|x}||_{L} \\
    & \leq \tilde{C} + \lambda^{\frac{\beta - \alpha}{2}}||k^{\alpha}||_{\infty}||C^{\beta}_{Y|X}|| \\
    & \leq (\tilde{C} + ||k^{\alpha}||_{\infty}||C^{\beta}_{Y|X}||)\lambda^{-\frac{(\alpha - \beta)_{+}}{2}}
\end{align*}
where we have used the fact that we may assume the fixed $\lambda \leq 1$ (as the $\lambda_n \to 0$ in Theorem \ref{Main Result}). Moreover, we have:
\begin{align}
    ||\mathbb{E}[(\mu_{Y|x} - \mu^{\lambda}_{Y|x}) \otimes (\mu_{Y|x} - \mu^{\lambda}_{Y|x})]||^2 &= \sup_{||f||_{L} \leq 1} \mathbb{E}[\langle f,  \mu_{Y|x} - \mu^{\lambda}_{Y|x}\rangle_{L}^2] \nonumber \\
    & = \sup_{||f||_{L} \leq 1} \mathbb{E}\Big[\Big(\sum_{i = 1}^{\infty} \frac{\lambda \cdot \mu^{\frac{\beta}{2}}_i e_i(X)}{\lambda + \mu_i} \langle f, C^{\beta}_{Y|X} \mu^{\frac{\beta}{2}}_i e_i \rangle_{L}\Big)^2\Big] \nonumber \\
    & \leq \sup_{||f||_{L} \leq 1} \sum_{i = 1}^{\infty} \Big(\frac{\lambda \cdot \mu^{\frac{\beta}{2}}_i}{\lambda + \mu_i}\Big)^2 \langle f, C^{\beta}_{Y|X} \mu^{\frac{\beta}{2}}_i e_i \rangle^2_{L} \label{eq: After Exp} \\
    & \leq \Big(\sup_{i} \Big(\frac{\lambda \cdot \mu^{\frac{\beta}{2}}_i}{\lambda + \mu_i}\Big)^2\Big) \sup_{||f||_{L} \leq 1} \sum_{i = 1}^{\infty} \langle f, C^{\beta}_{Y|X} \mu^{\frac{\beta}{2}}_i e_i \rangle^2_{L} \\
    & \leq \lambda^{\beta} ||C^{\beta}_{Y|X}||^2
\end{align}
where \eqref{eq: After Exp} follows from the fact that $\mathbb{E}_{X}[e_i(X)e_j(X)] = \delta_{ij}$ (as $\{e_i\}_{i = 1}^{\infty}$ is an orthonormal basis for $L^2(\nu)$), and the last step follows from $\{\mu_i^{\frac{\beta}{2}}e_i\}_{i = 1}^{\infty}$ being an orthonormal basis in $\mathcal{H}_{K}^{\beta}$. For the final part of Lemma \ref{orders}, we observe that like before:
\begin{align}
    C^{\lambda}_{Y|X} & = C_{YX}(C_{XX} + \lambda)^{-1} \nonumber \\
    & = C^{\beta}_{Y|X}\mathbb{E}_{X}[k^{\beta}(x, \cdot) \otimes k(x, \cdot)](C_{XX} + \lambda)^{-1} \nonumber \\
    & = \sum_{i = 1}^{\infty} \frac{\mu_i^{1 + \beta}}{\mu_i + \lambda} C^{\beta}_{Y|X} e_i \otimes e_i \label{eq: Normalized Formula}
\end{align}
Recall that:
\begin{equation*}
    ||C^{\lambda}_{Y|X} - C_{Y|X}||_{\gamma} = ||C^{\lambda}_{Y|X} \circ I^{*}_{1, \gamma, \nu} - C^{\beta}_{Y|X} \circ I^{*}_{\beta, \gamma, \nu}|| 
\end{equation*}
by definition (see remark after section \ref{CME Interpolation}). Now, observe that for any element $f = \sum_{i} a_i \mu^{\frac{\gamma}{2}}_i e_i \in \mathcal{H}^{\gamma}_K$ with $\{a_i\}_{i = 1}^{\infty} \in \ell^2$, we have that:
\begin{align}
    (C^{\lambda}_{Y|X} \circ I^{*}_{1, \gamma, \nu})f & = (C^{\lambda}_{Y|X} \circ I^{*}_{1, \gamma, \nu}) \Big(\sum_{i} a_i \mu^{\frac{\gamma}{2}}_i e_i\Big) \nonumber \\
    & = C^{\lambda}_{Y|X} \Big(\sum_{i} a_i \mu^{1 - \frac{\gamma}{2}}_i e_i\Big) \label{eq: Apply Adjoint Def} \\
    & = \Big(\sum_{i = 1}^{\infty} \frac{\mu_i^{1 + \beta}}{\mu_i + \lambda} C^{\beta}_{Y|X} e_i \otimes e_i\Big) \Big(\sum_{i} a_i \mu^{1 - \frac{\gamma}{2}}_i e_i\Big) \label{eq: Plug-In Normalized}\\
    & = \sum_{i = 1}^{\infty} \frac{a_i\mu_i^{1 + \beta -\frac{\gamma}{2}}}{\mu_i + \lambda} C^{\beta}_{Y|X} e_i  \label{eq: Apply Normalized}
\end{align}
where \eqref{eq: Apply Adjoint Def} follows from \eqref{Action of Adjoint}, \eqref{eq: Plug-In Normalized} follows from \eqref{eq: Normalized Formula} and noting that $\sum_{i} a_i \mu_i^{1 - \frac{\gamma}{2}} e_i \in \mathcal{H}_{K}$, since $\mu_i \to 0$ (as $C_{\nu}$ is compact) and $\frac{1 - \gamma}{2} > 0$ (as $\gamma < 1$ by assumption); and \eqref{eq: Apply Normalized} follows from noting that $\{\mu^{\frac{1}{2}}_i e_i\}_{i = 1}^{\infty}$ is an orthonormal basis in $\mathcal{H}_K$. Similarly, we have that:
\begin{align*}
    (C^{\beta}_{Y|X} \circ I^{*}_{\beta, \gamma, \nu})f & = (C^{\beta}_{Y|X} \circ I^{*}_{\beta, \gamma, \nu})\Big(\sum_{i} a_i \mu^{\frac{\gamma}{2}}_i e_i\Big) \\
    & = \sum_{i} a_i \mu^{\beta - \frac{\gamma}{2}}_i C^{\beta}_{Y|X}e_i
\end{align*}
Thus, we have that:
\begin{align*}
    ||C^{\lambda}_{Y|X} - C_{Y|X}||_{\gamma} & =||C^{\lambda}_{Y|X} \circ I^{*}_{1, \gamma, \nu} - C^{\beta}_{Y|X} \circ I^{*}_{\beta, \gamma, \nu}||  \\
    & = \sup_{||(a_i)_{i = 1}^{\infty}||_{\ell^2} = 1} ||(C^{\lambda}_{Y|X} \circ I^{*}_{1, \gamma, \nu} - C^{\beta}_{Y|X} \circ I^{*}_{\beta, \gamma, \nu}) \Big(\sum_{i} a_i \mu^{\frac{\beta}{2}}_i e_i\Big)||_{L} \\
    & = \sup_{||(a_i)_{i = 1}^{\infty}||_{\ell^2} = 1} \Big|\Big|\sum_{i = 1}^{\infty} \frac{a_i\lambda \cdot \mu^{\frac{\beta - \gamma}{2}}_i}{\mu_i + \lambda} \cdot C^{\beta}_{Y|X} \mu^{\frac{\beta}{2}}_i e_i \Big|\Big| \\
    & \leq  \Big(\sup_i \frac{\lambda \cdot \mu^{\frac{\beta - \gamma}{2}}_i}{\mu_i + \lambda}\Big) ||C^{\beta}_{Y|X}|| \\
    & \leq \lambda^{\frac{\beta - \gamma}{2}}||C^{\beta}_{Y|X}||
\end{align*}
\end{proof}

\begin{remark}[Expected Bias for Hilbert-Schmidt $C^{\beta}_{Y|X}$]
Observe that when $C^{\beta}_{Y|X}$ is Hilbert-Schmidt, we have, by the above proof:
\begin{align*}
    \mathbb{E}_{X}[||\mu_{Y|X} - \mu^{\lambda}_{Y|X}||^2_{L}] & = \mathbb{E}_{X}\Big[\Big|\Big|\sum_{i = 1}^{\infty} \frac{\lambda}{\lambda + \mu_i} C^{\beta}_{Y|X} \mu^{\beta}_i e_i(X)e_i\Big|\Big|^2_{L}\Big] \\
    & = \mathbb{E}_{X}\Big[\Big|\Big|\sum_{i = 1}^{\infty} \frac{\lambda \cdot \mu^{\frac{\beta}{2}}_i}{\lambda + \mu_i} C^{\beta}_{Y|X} \mu^{\frac{\beta}{2}}_i e_i(X)e_i\Big|\Big|^2_{L}\Big] \\
    & = \sum_{i = 1}^{\infty} \Big(\frac{\lambda \cdot \mu^{\frac{\beta}{2}}_i}{\lambda + \mu_i}\Big)^2 ||C^{\beta}_{Y|X} \mu^{\frac{\beta}{2}}_i e_i||^2_{L} \\
    & \leq \lambda^{\beta} ||C^{\beta}_{Y|X}||_{\text{HS}}
\end{align*}
\end{remark}
where the last line follows from Lemma A.1 in \cite{fischer2020sobolev} and the fact that $\{\mu^{\frac{\beta}{2}}_i e_i\}_{i = 1}^{\infty}$ is an orthonormal basis of $\mathcal{H}^{\beta}_K$. Thus, when $C^{\beta}_{Y|X}$ is Hilbert-Schmidt, we can achieve polynomial decay of the expected bias for all $\beta \in (0, 2)$. 
\begin{proof}[Proof of Theorem \ref{Embedding Concentration}]
\label{Main Result Proofs}
We begin like in the proof of Theorem 6.8 in \cite{fischer2020sobolev}. Namely, applying Lemma \ref{Operator Norms in Interpolation Spaces} we write:
\begin{align}
    ||\hat{C}_{Y|X} - C^{\lambda}_{Y|X}||_{\gamma} & = ||(\hat{C}_{Y|X} - C^{\lambda}_{Y|X}) \circ C^{\frac{1}{2}}_{1, \gamma, \nu}|| \label{eq: Apply Lemma 2} \\
    & = ||(\hat{C}_{Y|X} - C^{\lambda}_{Y|X}) \circ C^{\frac{1 - \gamma}{2}}_{XX}|| \label{eq: Equivalence of Covariance Operators} \\
    & = ||(\hat{C}_{YX}(\hat{C}_{XX} + \lambda)^{-1} - C_{YX}(C_{XX} + \lambda)^{-1}) C^{\frac{1 - \gamma}{2}}_{XX}|| \nonumber \\
    & \leq ||(\hat{C}_{YX} - C_{YX}(C_{XX} + \lambda)^{-1}(\hat{C}_{XX} + \lambda))(C_{XX} + \lambda)^{-\frac{1}{2}}|| \cdot \nonumber \\
    & ||(C_{XX} + \lambda)^{\frac{1}{2}}(\hat{C}_{XX} + \lambda)^{-1}(C_{XX} + \lambda)^{\frac{1}{2}}||||C^{\frac{1 - \gamma}{2}}_{XX}(C_{XX} + \lambda)^{-\frac{1}{2}}||\label{eq: Master Terms}
\end{align}
where \eqref{eq: Apply Lemma 2} follows from Lemma \ref{Operator Norms in Interpolation Spaces} and \eqref{eq: Equivalence of Covariance Operators} follows from the fact that $C_{1, \gamma, \nu} = C^{1 - \gamma}_{\nu} = C^{1 - \gamma}_{XX}$, since $C_{1, \gamma, \nu}$ has eigenfunctions $\{\mu_i^{\frac{1}{2}}e_i\}_{i = 1}^{\infty}$ and eigenvalues $\{\mu^{1 - \gamma}_i\}_{i = 1}^{\infty}$ (see proof of Lemma \ref{Operator Norms in Interpolation Spaces} in Appendix \ref{Assumptions Proofs}). Note here, we have used the notation $C_{XX}$ instead of $C_{\nu}$ to remain consistent with the expansions of $\hat{C}_{Y|X}$ and $C^{\lambda}_{Y|X}$ in the literature. We primarily focus on bounding the first factor on the RHS of \eqref{eq: Master Terms}, as the remaining factors can be estimated simply as discussed previously in \cite{fischer2020sobolev}. To start, we again imitate the approach from the proof of Theorem 6.8 in \cite{fischer2020sobolev}. Namely, we have:
\begin{align*}
    \hat{C}_{YX} - C_{YX}(C_{XX} + \lambda)^{-1}(\hat{C}_{XX} + \lambda) &= \hat{C}_{YX} - C_{YX}(C_{XX} + \lambda)^{-1}(C_{XX} + \lambda + \hat{C}_{XX} - C_{XX}) \\
    & = \hat{C}_{YX} - C_{YX} + C_{YX}(C_{XX} + \lambda)^{-1}(C_{XX} - \hat{C}_{XX}) \\
    & = \hat{C}_{YX} - C_{YX}(C_{XX} + \lambda)^{-1}\hat{C}_{XX} - (C_{YX} - C_{YX}(C_{XX} + \lambda)^{-1}C_{XX}) \\
    & = \hat{C}_{YX} - C_{YX}(C_{XX} + \lambda)^{-1}\hat{\mathbb{E}}[k(X, \cdot) \otimes k(X, \cdot)] - C_{YX} \\
    & - C_{YX}(C_{XX} + \lambda)^{-1}\mathbb{E}[k(X, \cdot) \otimes k(X, \cdot)] \\
    & = \hat{\mathbb{E}}[(L(Y, \cdot) - \mu^{\lambda}_{Y|X}) \otimes k(X, \cdot)] - \mathbb{E}[(L(Y, \cdot) - \mu^{\lambda}_{Y|X}) \otimes k(X, \cdot)] \\
\end{align*}
We now wish to apply Lemma \ref{General Bernstein} to bound this deviation. Let $h(X, \cdot) = (C_{XX} + \lambda)^{-\frac{1}{2}}k(X, \cdot)$. We first write:
\begin{equation*}
    (L(Y, \cdot) - \mu^{\lambda}_{Y|X}) \otimes h(X, \cdot) = (L(Y, \cdot) - \mu_{Y|X}) \otimes h(X, \cdot) + (\mu_{Y|X}- \mu^{\lambda}_{Y|X}) \otimes h(X, \cdot)
\end{equation*}
Then, applying Corollary \ref{Matrix Convexity}, we can write:
\begin{small}
\begin{align}
    \Big[\Big((L(Y, \cdot) - \mu^{\lambda}_{Y|X}) \otimes h(X, \cdot)\Big)^{*}\Big((L(Y, \cdot) - \mu^{\lambda}_{Y|X}) \otimes h(X, \cdot) \Big)\Big]^{p} & \preccurlyeq 2^{2p - 1}\Big[||L(Y, \cdot) - \mu_{Y|X}||_{L}^{2p} \Big(h(X, \cdot) \otimes h(X, \cdot)\Big)^{p} \label{eq: Term1} \\
    & + ||\mu^{\lambda}_{Y|X} - \mu_{Y|X}||_{L}^{2p} \Big(h(X, \cdot) \otimes h(X, \cdot)\Big)^{p}\Big] \label{eq: Term2} \\
    \Big[\Big((L(Y, \cdot) - \mu^{\lambda}_{Y|X}) \otimes h(X, \cdot)\Big)\Big((L(Y, \cdot) - \mu^{\lambda}_{Y|X}) \otimes h(X, \cdot) \Big)^{*}\Big]^{p} & \preccurlyeq 2^{2p - 1}||h(X, \cdot)||_{K}^{2p}\Big[\Big((L(Y, \cdot) - \mu_{Y|X}) \otimes (L(Y, \cdot) - \mu_{Y|X})\Big)^p  \label{eq: Term3} \\
    & + \Big((\mu^{\lambda}_{Y|X} - \mu_{Y|X}) \otimes (\mu^{\lambda}_{Y|X} - \mu_{Y|X})\Big)^p\Big] \label{eq: Term4}
\end{align}
\end{small}
Hence, we have four terms to consider. We begin first with the RHS of \eqref{eq: Term1}:
\begin{align}
    \mathbb{E}[||L(Y, \cdot) - \mu_{Y|X}||_{L}^{2p} \Big(h(X, \cdot) \otimes h(X, \cdot)\Big)^{p}] & \preccurlyeq \mathbb{E}[||L(Y, \cdot) - \mu_{Y|X}||_{L}^{2p} ||h(X, \cdot)||_{K}^{2(p - 1)}\Big(h(X, \cdot) \otimes h(X, \cdot)\Big)] \nonumber \\
    & = \mathbb{E}_{X}\Big[\mathbb{E}_{Y|X}\Big[||L(Y, \cdot) - \mu_{Y|X}||_{L}^{2p}\Big] \cdot  ||h(X, \cdot)||_{K}^{2(p - 1)}\Big(h(X, \cdot) \otimes h(X, \cdot)\Big)\Big] \nonumber \\
    & \preccurlyeq \frac{R^{2p -  2}(2p)!\sigma^2}{2}\mathbb{E}_{X}\Big[||h(X, \cdot)||_K^{2(p - 1)}\Big(h(X, \cdot) \otimes h(X, \cdot)\Big)\Big] \label{eq: Subexponential Application} \\
    & \preccurlyeq \frac{||k^{\alpha}||^{2(p - 1)}_{\infty}R^{2p -  2}(2p)!\sigma^2}{2\lambda^{\alpha(p - 1)}} \mathbb{E}_X\Big[h(X, \cdot) \otimes h(X, \cdot)\Big] \label{eq: h bound} \\
    & = \frac{||k^{\alpha}||^{2(p-1)}_{\infty}R^{2p -  2}(2p)!\sigma^2}{2\lambda^{\alpha(p-1)}}C_{XX}(C_{XX} + \lambda)^{-1} \label{eq: Original Power}
\end{align}
where we have taken the trace of both sides in Assumption \ref{Subexponential} to obtain \eqref{eq: Subexponential Application} and have applied Lemma \ref{EMB to h-bound} to obtain \eqref{eq: h bound}. By a similar reasoning, we have:
\begin{align}
    \mathbb{E}\Big[||h(X, \cdot)||^{2p}_{L} \Big((L(Y, \cdot) - \mu_{Y|X}) \otimes (L(Y, \cdot) - \mu_{Y|X})\Big)^{p}\Big] & \preccurlyeq \frac{||k^{\alpha}||^{2p-2}_{\infty}R^{2p -  2}(2p)!\mathcal{N}(\lambda)V}{2\lambda^{(p-1)\alpha}} \label{eq: Adjoint Power}
\end{align}
for the RHS of \eqref{eq: Term3} after again applying Assumption \ref{Subexponential}. Note the only difference between the adjoint moment in \eqref{eq: Adjoint Power} and \eqref{eq: Original Power} is that we have taken the trace of $C_{XX}(C_{XX} + \lambda)^{-1}$ ($\mathcal{N}(\lambda)$) in the former instead of $\text{tr}(V) = \sigma^2$). Now, for \eqref{eq: Term2}, we have:
\begin{align*}
    \mathbb{E}\Big[\Big(||\mu_{Y | X} - \mu^{\lambda}_{Y | X}||^2_L h(X, \cdot) \otimes h(X, \cdot) \Big)^p\Big] & = \mathbb{E}\Big[||\mu_{Y | X} - \mu^{\lambda}_{Y | X}||^{2p}||h(X, \cdot)||_{K}^{2p - 2}\Big(h(X, \cdot) \otimes h(X, \cdot)\Big)\Big] \\
    & \preccurlyeq \frac{M(\lambda)^{2p}||k^{\alpha}||_{\infty}^{2p - 2}}{\lambda^{(p - 1)\alpha}}C_{XX}(C_{XX} + \lambda)^{-1} \\
    & \preccurlyeq \frac{(2p)!(M(\lambda))^{2p}||k^{\alpha}||_{\infty}^{2p - 2}}{2\lambda^{(p - 1)\alpha}}C_{XX}(C_{XX} + \lambda)^{-1} 
\end{align*}
where we recall the definition of $M(\lambda)$ from Lemma \ref{orders}. Finally for \eqref{eq: Term4}
\begin{small}
\begin{align*}
    \mathbb{E}\Big[||\mu^{\lambda}_{Y|X} - \mu_{Y|X}||^{2(p - 1)}_L ||h(X, \cdot)||_K^{2p} \Big((\mu^{\lambda}_{Y|X} - \mu_{Y|X}) \otimes (\mu^{\lambda}_{Y|X} - \mu_{Y|X})\Big)\Big] & \preccurlyeq \frac{(2p)! M(\lambda)^{2(p - 1)} ||k^{\alpha}||^{2p}_{\infty}}{2\lambda^{p\alpha}} \cdot \\
    & \mathbb{E}\Big[(\mu^{\lambda}_{Y|X} - \mu_{Y|X}) \otimes (\mu^{\lambda}_{Y|X} - \mu_{Y|X})\Big]
\end{align*}
\end{small}
Let $Q = M(\lambda) \vee R$ and $\rho_{\lambda} = \mathbb{E}\Big[(\mu_{Y | X} - \mu^{\lambda}_{Y | X}) \otimes (\mu_{Y | X} - \mu^{\lambda}_{Y | X})\Big]$. Then, we can apply Lemma \ref{General Bernstein} with $\tilde{V} = 2(\sigma^2 + M^2(\lambda))C_{XX}(C_{XX} + \lambda)^{-1}$, $\tilde{W} = 2\mathcal{N}(\lambda)V + \frac{2||k^{\alpha}||^2_{\infty}}{\lambda^{\alpha}}p_{\lambda}$. Then, we have that, with probability $1 - \delta$:
\begin{equation*}
    ||(\hat{C}_{YX} - C_{YX}(C_{XX} + \lambda)^{-1}(\hat{C}_{XX} + \lambda))(C_{XX} + \lambda)^{-\frac{1}{2}}|| \leq \frac{16Q||k^{\alpha}||_{\infty}\beta(\delta)}{\lambda^{\frac{\alpha}{2}}n} + 8\sqrt{\frac{\eta\beta(\delta)}{n}}
\end{equation*}
where:
\begin{align*}
    \rho_{\lambda} & = \mathbb{E}\Big[(\mu^{\lambda}_{Y|X} - \mu_{Y|X}) \otimes (\mu^{\lambda}_{Y|X} - \mu_{Y|X})\Big] \\
    \eta & = \max\{(\sigma^2 + M^2(\lambda))||C_{\nu}||(||C_{\nu}|| + \lambda)^{-1}, ||\mathcal{N}(\lambda)V + \frac{||k^{\alpha}||^2_{\infty}}{\lambda^{\alpha}}\rho_{\lambda}||\} \\
    \beta(\delta) & = \log \Big(\frac{4((\sigma^2 + M^2(\lambda))\mathcal{N}(\lambda) + (\sigma^2 \mathcal{N}(\lambda) + \frac{||k^{\alpha}||^2_{\infty}}{\lambda^{\alpha}}\mathbb{E}_X[||\mu^{\lambda}_{Y|X} - \mu_{Y|X}||_L^2]))}{\eta \delta}\Big)
\end{align*}
The last term in \eqref{eq: Master Terms} is bounded as follows:
\begin{align*}
    ||C^{\frac{1 -\gamma}{2}}_{XX}(C_{XX} + \lambda)^{-\frac{1}{2}}|| \leq \sqrt{\sup_i \frac{\mu_i^{1 - \gamma}}{\mu_i + \lambda}} \leq \lambda^{-\frac{\gamma}{2}}
\end{align*}
Finally, for the middle term, we may follow the proof of Theorem 6.8 in \cite{fischer2020sobolev} exactly to obtain:
\begin{equation*}
    ||(C_{XX} + \lambda)^{\frac{1}{2}}(\hat{C}_{XX} + \lambda)^{-1}(C_{XX} + \lambda)^{\frac{1}{2}}|| \leq 3
\end{equation*}
for $n \geq 8||k^{\alpha}||^2_{\infty}\log(\delta^{-1}) g_{\lambda}\lambda^{-\alpha}$ with probability $1 - \delta$ (for brevity, we do not repeat this argument here) . Putting these together, we obtain our result. 
\end{proof}
\begin{proof}[Proof of Theorem \ref{Main Result}]
We must first demonstrate there exists a $n_0 \in \mathbb{N}$, such that for all $n \geq n_0$, $n \geq 8||k^{\alpha}||^2_{\infty}\log(\delta^{-1}) g_{\lambda_n}\lambda_n^{-\alpha}$ in order apply the result in Theorem \ref{Embedding Concentration}. Since $\lambda_n \to 0$, we can let $\lambda_n \leq \min\{1, ||C_{\nu}||\}$, from which we obtain:
\begin{align}
    \frac{8||k^{\alpha}||^2_{\infty}\log(\delta^{-1}) g_{\lambda_n}\lambda_n^{-\alpha}}{n} & = \frac{8||k^{\alpha}||^2_{\infty}\log(\delta^{-1}) \lambda_n^{-\alpha}}{n} \cdot \log \Big(2e\mathcal{N}(\lambda_n)\frac{||C_{\nu}|| + \lambda_n}{||C_{\nu}||}\Big) \nonumber \\
    & \leq \frac{8||k^{\alpha}||^2_{\infty}\log(\delta^{-1}) \lambda_n^{-\alpha}}{n} \cdot \log 4M_1e\lambda_n^{-p^{-1}} \label{eq: Apply Effective Dimension}\\
    & = \frac{8||k^{\alpha}||^2_{\infty}\log 4M_1e \cdot \log(\delta^{-1}) \lambda_n^{-\alpha}}{n} + \frac{8p^{-1}||k^{\alpha}||^2_{\infty}\log(\delta^{-1}) \lambda_n^{-\alpha}\log \lambda^{-1}_n}{n} \nonumber
\end{align}
where \eqref{eq: Apply Effective Dimension} follows from Lemma \ref{Effective Dimension Bound}. Thus, in order to demonstrate $\frac{8||k^{\alpha}||^2_{\infty}\log(\delta^{-1}) g_{\lambda_n}\lambda_n^{-\alpha}}{n} \to 0$, it is sufficient to show $\frac{\lambda_n^{-\alpha}\log \lambda^{-1}_n}{n} \to 0$. This follows from the fact that:
\begin{equation*}
    \frac{\lambda_n^{-\alpha}\log \lambda^{-1}_n}{n} \asymp \frac{(\log n)^{1 -  \frac{r\alpha}{\max \{\alpha, \beta + p\}}}}{n^{1 -  \frac{\alpha}{\max \{\alpha, \beta + p\}}}}
\end{equation*}
after substituting for $\lambda_n$, and observing that $\frac{(\log n)^{1 -  \frac{r\alpha}{\max \{\alpha, \beta + p\}}}}{n^{1 -  \frac{\alpha}{\max \{\alpha, \beta + p\}}}} \to 0$ as $n \to \infty$ since $r > 1$. We now estimate each term in \eqref{Variance Bound}. We first have that:
\begin{align}
    \beta(\delta) & = \log \Big(\frac{4((\sigma^2 + M^2(\lambda_n))\mathcal{N}(\lambda_n) + \sigma^2 \mathcal{N}(\lambda_n) + \frac{||k^{\alpha}||^2_{\infty}}{\lambda^{\alpha}}\mathbb{E}_X[||\mu^{\lambda}_{Y|X} - \mu_{Y|X}||_L^2])}{\eta \delta}\Big) \nonumber \\
    & \leq \log \Big(\frac{4(\sigma^2 + M^2(\lambda_n))\mathcal{N}(\lambda_n)}{(\sigma^2 + M^2(\lambda))||C_{\nu}||(||C_{\nu}|| + \lambda)^{-1}} + \frac{4\sigma^2 \mathcal{N}(\lambda_n) + \frac{4||k^{\alpha}||^2_{\infty}}{\lambda^{\alpha}}\mathbb{E}_X[||\mu^{\lambda}_{Y|X} - \mu_{Y|X}||_L^2]}{||\mathcal{N}(\lambda_n)V + \frac{||k^{\alpha}||^2_{\infty}}{\lambda_n^{\alpha}}\rho_{\lambda_n}||}\Big) - \log \delta \label{eq: ED Sublinear} \\
    & \leq \log \Big(\frac{4\sigma^2 \mathcal{N}(\lambda_n)}{\mathcal{N}(\lambda_n)||V||} + \frac{4\mathcal{N}(\lambda_n)}{||C_{\nu}||(||C_{\nu}|| + \lambda)^{-1}} + \frac{\frac{4||k^{\alpha}||^2_{\infty}}{\lambda^{\alpha}}\mathbb{E}_X[||\mu^{\lambda}_{Y|X} - \mu_{Y|X}||_L^2]}{\mathcal{N}(\lambda_n)||V||}\Big) - \log \delta \label{eq: Denom Split} \\
    & \leq \log \Big(N_1\lambda^{-p}_n + N_2\lambda^{\beta - \alpha}_n\Big) - \log \delta \label{eq: Apply Orders}
\end{align}
for $N_1 = \frac{4M_1||C_{\nu} + \lambda||}{||C_{\nu}||} + \frac{4\sigma^2}{||V||}$ and $N_2 = \frac{4||k^{\alpha}||^2_{\infty}DB^2}{M_2 ||V||}$. Note, \eqref{eq: ED Sublinear} follows from the definition of $\eta$ and the fact that $\frac{\text{tr}(A + B)}{\max\{||A||, ||B||\}} \leq \frac{\text{tr}(A)}{||A||} + \frac{\text{tr}(B)}{||B||}$ for any self-adjoint operators $A$ and $B$; \eqref{eq: Denom Split} follows from the sublinearity of the operator norm and the fact that $\rho_{\lambda_n} \succcurlyeq	0$; and \eqref{eq: Apply Orders} follows from applying Lemmas \ref{Effective Dimension Bound}, \ref{Effective Dimension Lower Bound}, \eqref{eq: Expected Deviation}, and noting that we can restrict $\lambda_n \leq 1$ (which allows the absorption of the constant term $\frac{4\sigma^2 }{||V||}$ into $N_1$). Thus, it follows that $\beta(\delta) \leq N_3\log (\delta^{-1} \cdot \lambda_n^{-\max\{\alpha - \beta, p\}})$ for some $N_3 > 0$. Moreover, we have that:
\begin{align}
    Q & = M(\lambda) \vee R \nonumber \\
    & \leq (\tilde{C} + ||k^{\alpha}||_{\infty}B)\lambda_n^{-\frac{(\alpha - \beta)_{+}}{2}} \vee R \nonumber \\
    & \leq N_4\lambda_n^{-\frac{(\alpha - \beta)_{+}}{2}} \label{eq: Estimate for Q}
\end{align}
for $N_4 = \max\{\tilde{C} + ||k^{\alpha}||_{\infty}B, R\}$ (where the penultimate step follows from applying \eqref{eq: Max Deviation} and the final step follows since we can again assume $\lambda_n \leq 1$). We also have:
\begin{align}
    \eta & = \max\{(\sigma^2 + M^2(\lambda))||C_{\nu}||(||C_{\nu}|| + \lambda)^{-1}, ||\mathcal{N}(\lambda)V + \frac{||k^{\alpha}||^2_{\infty}}{\lambda^{\alpha}}\rho_{\lambda}||\}\nonumber \\
    & \leq \max\{(\sigma^2 + N^2_4 \lambda^{-(\alpha - \beta)_{+}}_n)||C_{\nu}||(||C_{\nu}|| + \lambda)^{-1}, M_1\lambda_n^{-p}||V|| + B^2||k^{\alpha}||^2_{\infty}\lambda^{-\alpha}_n \lambda^{\beta}_n\}  \label{eq: Using Orders} \\
    & \leq N_5 \lambda_n^{-\max\{p, \alpha - \beta\}} \label{eq: Estimate for Eta}
\end{align}
for some $N_5 > 0$. Note, in \eqref{eq: Using Orders}, we have applied \eqref{eq: Max Deviation}, \eqref{eq: Expected Deviation Norm}, and Lemma \ref{Effective Dimension Bound}. Thus, we have that:
\begin{align}
    ||\hat{C}_{Y|X} - C^{\lambda}_{Y|X}||_{\gamma} & \leq 3\lambda_n^{-\frac{\gamma}{2}}\Big(\frac{16Q||k^{\alpha}||_{\infty}\beta(\delta)}{\lambda_n^{\frac{\alpha}{2}}n} + 8\sqrt{\frac{\eta\beta(\delta)}{n}}\Big) \nonumber \\
    & \leq 24\lambda_n^{-\frac{\gamma}{2}}\Big(\frac{2N_4||k^{\alpha}||_{\infty}\beta(\delta)}{n\lambda_n^{\frac{\alpha + (\alpha - \beta)_{+}}{2}}} + \sqrt{\frac{N_5\beta(\delta)}{n\lambda_n^{\max\{p, \alpha - \beta\}}}}\Big) \label{eq: Plug-In Estimates}\\
    & \leq 24\lambda_n^{-\frac{\gamma}{2}}\sqrt{\frac{\beta(\delta)}{n\lambda_n^{\max\{p, \alpha - \beta\}}}} \Big(2N_4||k^{\alpha}||_{\infty}\sqrt{\frac{\beta(\delta)}{n\lambda_n^{\alpha + (\alpha - \beta)_{+} - \max\{p, \alpha - \beta\}}}} + N_5\Big) \nonumber
\end{align}
where \eqref{eq: Plug-In Estimates} follows from \eqref{eq: Estimate for Eta} and \eqref{eq: Estimate for Q}. Then, like in the proof of Theorem 6.8 in \cite{fischer2020sobolev}, we consider the inner factor for our two parameter regimes. When $p < \alpha - \beta$, we have that $\lambda_n \asymp \Big(\frac{n}{\log^r n}\Big)^{-\frac{1}{\alpha}}$, and thus:
\begin{equation*}
    \frac{\beta(\delta)}{n\lambda_n^{\alpha + (\alpha - \beta)_{+} - \max\{p, \alpha - \beta \}}} \leq \frac{N_3\log (\delta^{-1} \cdot \lambda_n^{\beta - \alpha})}{n\lambda^{\alpha}_n} = \log(\delta^{-1}) \cdot \mathcal{O}\Big(\frac{\log n}{\log^{r} n}\Big)
\end{equation*}
Thus, since $r > 1$, it follows that $\frac{\beta(\delta)}{n\lambda_n^{\alpha + (\alpha - \beta)_{+} - \max\{p, \alpha - \beta \}}} \to 0$ when $p < \alpha - \beta$. Similarly, when $p > \alpha - \beta$, we have that $\lambda_n \asymp \Big(\frac{n}{\log^r n}\Big)^{-\frac{1}{\beta + p}}$, and:
\begin{equation*}
    \frac{\beta(\delta)}{n\lambda_n^{\alpha + (\alpha - \beta)_{+} - \max\{p, \alpha - \beta\}}} \leq \frac{N_3\log (\delta^{-1} \cdot \lambda_n^{-p})}{n\lambda_n^{\alpha + (\alpha - \beta)_{+} - p}} = \log(\delta^{-1}) \cdot \mathcal{O}\Big((\log n)^{1 - \frac{r\alpha + r(\alpha - \beta)_{+} - rp}{\beta + p}}n^{-\Big(1 - \frac{\alpha + (\alpha - \beta)_{+} - p}{\beta + p}\Big)}\Big)
\end{equation*}
Thus, since $1 - \frac{\alpha + (\alpha - \beta)_{+} - p}{\beta + p} > 0$ by the assumption that $p > \alpha - \beta$, we again have $\frac{\beta(\delta)}{n\lambda_n^{\alpha + (\alpha - \beta)_{+} - \max\{p, \alpha - \beta \}}} \to 0$ when $p > \alpha - \beta$. Hence,  we can bound, $\sqrt{\frac{\beta(\delta)}{n\lambda_n^{\alpha + (\alpha - \beta)_{+} - \max\{p, \alpha - \beta\}}}}$ by $N^2_6 \log (\delta^{-1})$ for some constant $N_6 > 0$.  Thus, putting this all together and combining with the bias bound $||C^{\lambda}_{Y|X} - C_{Y|X}||_{\gamma} \leq B\lambda^{\frac{\beta - \gamma}{2}}$ in \eqref{eq: CME Deviation Norm}, we obtain by \eqref{Bias-Variance Breakdown}:
\begin{align*}
    ||\hat{C}_{Y|X} - C_{Y|X}||_{\gamma} & \leq ||\hat{C}_{Y|X}  - C^{\lambda}_{Y|X}||_{\gamma} + ||C^{\lambda}_{Y|X} - C_{Y | X}||_{\gamma}\\
    & \leq 24\lambda_n^{-\frac{\gamma}{2}}\sqrt{\frac{\beta(\delta)}{n\lambda_n^{\max\{p, \alpha - \beta\}}}} \Big(2N_4||k^{\alpha}||_{\infty}\sqrt{\frac{\beta(\delta)}{n\lambda_n^{\alpha + (\alpha - \beta)_{+} - \max\{p, \alpha - \beta\}}}} + N_5\Big) + B\lambda^{\frac{\beta - \gamma}{2}} \\
    & \leq 24\lambda_n^{-\frac{\gamma}{2}}\sqrt{\frac{\log(\delta^{-1}) \cdot \beta(\delta)}{n\lambda_n^{\max\{p, \alpha - \beta\}}}} \Big(2N_4N_6||k^{\alpha}||_{\infty} + N_5\Big) + B\lambda^{\frac{\beta - \gamma}{2}} \\
    & = \lambda_n^{\frac{\beta - \gamma}{2}}\Big(N_7\sqrt{\frac{\log(\delta^{-1})\beta(\delta)}{n\lambda_n^{\max\{\beta + p, \alpha\}}}} + B\Big)
\end{align*}
where we have set $N_7 = 24(2N_4N_6||k^{\alpha}||_{\infty} + N_5)$. Now, noting that $\lambda_n \asymp \Big(\frac{n}{\log^r n}\Big)^{-\frac{1}{\max\{\alpha, \beta + p\}}}$ by definition, we observe that:
\begin{align}
    \frac{\beta(\delta)}{n\lambda_n^{\max\{\beta + p, \alpha\}}} & \leq \frac{\log (\delta^{-1} \cdot \lambda_n^{-\max\{\alpha - \beta, p\}})}{n\lambda_n^{\max\{\beta + p, \alpha \}}} \nonumber \\
    & \leq \frac{\log (\delta^{-1}) + \log (\lambda_n^{-\max\{\alpha - \beta, p\}})}{n\lambda_n^{\max\{\beta + p, \alpha \}}} \nonumber \\
    & \leq \frac{\log (\delta^{-1}) + \log (\lambda_n^{-\max\{\alpha, \beta + p\}})}{n\lambda_n^{\max\{\beta + p, \alpha \}}} \label{eq: lambda small} \\
    & \leq \frac{\log (\delta^{-1}) + \log n - \log \log^r n}{n \cdot \frac{\log^r n}{n}} \nonumber \\
    & \leq \frac{\log(\delta^{-1}) + \log n}{\log^r n} \nonumber 
\end{align}
where \eqref{eq: lambda small} follows from the fact that $\lambda^{-\beta}_n \geq 1$ as $\lambda_n \to 0$. Hence, since $\delta < 1$ and $r > 1$, we have that $\frac{\beta(\delta)}{n\lambda_n^{\max\{\beta + p, \alpha\}}} = \mathcal{O}(\log(\delta^{-1}))$ as $n \to \infty$. Thus, we have, that there exists a $K > 0$ not depending on $n$ or $\delta$, such that:
\begin{equation*}
    ||\hat{C}_{Y|X} - C_{Y|X}||_{\gamma} \leq K \log (\delta^{-1})\lambda_n^{\frac{\beta - \gamma}{2}}
\end{equation*}
with probability $1 - 2\delta$. 
\end{proof}
\section{Concentration Bounds}
\label{Concentration Proofs}
\begin{lemma}
\label{Symmetric Bernstein}
Let $X_{1}, X_{2}, \ldots X_{N}$ be i.i.d self-adjoint operators on a Hilbert space $\mathcal{V}$, with:
\begin{align*}
    E[X_i] & = 0 \\
    E[X^{2p}_i] & \preccurlyeq \frac{R^{2p - 2}(2p)!}{2}V  \hspace{3mm} \forall p \in \mathbb{N} \\
    ||V|| &= \sigma^2
\end{align*}
where $V$ is a trace-class operator. Let $\delta > 0$ and $\beta(\delta) = \log \Big(\frac{4\text{tr}(V)}{\delta\sigma^2}\Big)$. Then, for $t \geq \frac{2R}{N} + \frac{2^{\frac{3}{4}}\sigma}{\sqrt{N}}$ we have that:
\begin{equation*}
    \Big|\Big|\frac{1}{N}\sum_{i = 1}^N X_i\Big|\Big| \leq \frac{4R\beta(\delta)}{N} + 2\sigma\sqrt{\frac{2\beta(\delta)}{N}}
\end{equation*}
with probability $1 - \delta$. 
\end{lemma}
\begin{proof}
We first note that for odd $p \geq 1$ and $y \in \mathcal{V}$, we have $\mathbb{E}[\langle y, X^{p}_i y \rangle_{\mathcal{V}}] = \mathbb{E}[\langle X_i y, X^{p -  1}_i y \rangle_{\mathcal{V}}] \leq \sqrt{\mathbb{E}[\langle y, X^2_i y \rangle_{\mathcal{V}}]\mathbb{E}[\langle y, X^{2p -2}_i y \rangle_{\mathcal{V}}]} \leq \sqrt{\frac{R^{2p - 4}(2p - 2)!}{2}}\langle y, Vy \rangle \leq \frac{(2R)^{p - 2}(p -  1)!}{2}\langle y, \sqrt{8}Vy \rangle$. Thus, letting $S = 2R$, we have, by the usual construction:
\begin{align}
    \mathbb{E}[e^{\theta X}] & = I + \sum_{j = 2}^{\infty} \frac{\mathbb{E}[(\theta X)^{j}]}{j!} \nonumber \\
    & \preccurlyeq  I + \sum_{j = 2}^{\infty} \frac{\sqrt{8}(\theta S)^{j}V}{2S^2} \nonumber \\
    & = I +\frac{\sqrt{8}\theta^2 V}{2} \sum_{j = 0}^{\infty}  (\theta S)^j \nonumber \\
    & = I +\frac{\sqrt{8}\theta^2 V}{2(1 -  \theta S)} \nonumber \\
    & \preccurlyeq \text{exp}\Big(\frac{\sqrt{8}\theta^2 V}{2(1 -  \theta S)}\Big) \label{eq: Laplace Bound}
\end{align}
where the first equality follows by assumption. Let $g(\theta) = \frac{\theta^2}{2(1 -  \theta S)}$. Equipped with this result, we then have:
\begin{align}
    P\Big(\Big|\Big|\frac{1}{N}\sum_{i = 1}^N X_i\Big|\Big| > t \Big) & \leq \frac{\mathbb{E}\Big[\Big|\Big|e^{\frac{\theta}{N} \sum_{i = 1}^N X_i} - \frac{\theta}{N} \sum_{i = 1}^N X_i - I\Big|\Big|\Big]}{e^{\theta t} - \theta t - 1} \nonumber \\
    & \leq \frac{\mathbb{E}[\text{tr}(e^{\frac{\theta}{N} \sum_{i = 1}^N X_i} - I)]}{e^{\theta t} - \theta t - 1} \nonumber \\
    & \leq \frac{\text{tr}(\text{exp}\Big(\sum_{i = 1}^N \log \mathbb{E}[e^{\frac{\theta}{N}  X_i}]\Big) - I)}{e^{\theta t} - \theta t - 1} \label{eq: Operator Concavity} \\
    & \leq \frac{\text{tr}(e^{\sqrt{8}N g(N^{-1}\theta)V} - I)}{e^{\theta t} - \theta t - 1} \label{eq: Laplace Bound Plug-In}\\
    & \leq \frac{\text{tr}(V)}{||V||} \cdot \frac{e^{\sqrt{8}N g(N^{-1}\theta)||V||} - 1}{e^{\theta t} - \theta t - 1} \label{eq: Trace-to-Norm}\\
    & \leq \frac{\text{tr}(V)}{||V||} \cdot \frac{e^{\theta t}e^{\sqrt{8}N g(N^{-1} \theta)||V|| - \theta t}}{e^{\theta t} - \theta t - 1} \nonumber
\end{align}
where \eqref{eq: Operator Concavity} follows from the iterative application of the operator concavity of $\text{tr}(\text{exp}(A + \log X))$ in $X$ (see e.g. \cite{tropp2015introduction}), \eqref{eq: Laplace Bound Plug-In} follows from applying \eqref{eq: Laplace Bound}, and \eqref{eq: Trace-to-Norm} follows from Lemma 7.5.1 in \cite{tropp2015introduction} and the observation that $f(t) = e^{\theta t} - 1$ is convex with $f(0) = 0$. Applying the bound $\frac{e^{a}}{e^{a} - a - 1} \leq 1 + \frac{3}{a^2}$ for $a \geq 0$ (see e.g. the proof of Theorem 7.7.1 in \cite{tropp2015introduction}), we obtain:
\begin{align*}
    P\Big(\Big|\Big|\frac{1}{N}\sum_{i = 1}^N X_i\Big|\Big| > t \Big) & \leq \frac{\text{tr}(V)}{\sigma^2}\Big(1 + \frac{3}{\theta^2 t^2}\Big)e^{\sqrt{8}N g(N^{-1} \theta)\sigma^2 - \theta t} \\
    & \leq \frac{\text{tr}(V)}{\sigma^2}\Big(1 + \frac{3(\sqrt{8}\sigma^2 + 2Rt)^2}{N^2 t^4}\Big)\text{exp}\Big(-\frac{N t^2}{2(\sqrt{8}\sigma^2 + 2Rt)}\Big)
\end{align*}
after setting $\theta = \frac{Nt}{\sqrt{8}\sigma^2 + 2Rt}$, noting $S = 2R$, and observing that $\sqrt{8}N \sigma^2 g(N^{-1}\theta) - \theta t = \frac{N t^2}{2(\sqrt{8}\sigma^2 + 2Rt)} - \frac{N t^2}{\sqrt{8}\sigma^2 + 2Rt} \leq - \frac{N t^2}{2(\sqrt{8}\sigma^2 + 2Rt)}$. We consider only the case where $N t^2 \geq \sqrt{8}\sigma^2 + 2Rt$, noting like in Theorem 7.7.1 in \cite{tropp2015introduction} that the Chernoff bound above is typically vacuous when this restriction is violated. Solving this quadratic inequality,  we obtain the more amenable expression: $t \geq \frac{R}{N} + \sqrt{\frac{R^2}{N^2} + \frac{\sqrt{8}\sigma^2}{N}}$. Thus, applying the fact that $\sqrt{a + b} \leq \sqrt{a} + \sqrt{b}$, we have that for $t \geq \frac{2R}{N} + \frac{8^{0.25}\sigma}{\sqrt{N}}$:
\begin{equation*}
    P\Big(\Big|\Big|\frac{1}{N}\sum_{i = 1}^N X_i\Big|\Big| > t \Big) \leq \frac{4\text{tr}(V)}{\sigma^2}\text{exp}\Big(-\frac{N t^2}{2(\sqrt{8}\sigma^2 + 2Rt)}\Big)
\end{equation*}
The result follows from setting the RHS equal to $\delta$, solving for $t$ using the quadratic formula, applying the triangle inequality to this solution, and noting that $\sqrt{2} \leq 2$, we obtain our result. 
\end{proof}
\begin{remark}
We emphasize the qualification $t \geq \frac{2R}{N} + \frac{2^{\frac{3}{4}}\sigma}{\sqrt{N}}$ in Lemma \ref{Symmetric Bernstein} is not very restrictive as the derived Chernoff bound is typically vacuous when this restriction is violated (and therefore is avoided by choosing sufficiently small $\delta$. For brevity, we therefore omit this restriction in the below generalizations). 
\end{remark}
\begin{lemma}
\label{Rectangular Bernstein}
Let $X_{1}, X_{2}, \ldots X_{N}$ be i.i.d operators from $\mathcal{V}$ to $\mathcal{W}$, with:
\begin{align*}
    E[X_i] & = 0 \\
    E[(X_i^* X_i)^{p}] & \preccurlyeq \frac{R^{2p - 2}(2p)!}{2}V  \hspace{3mm} \forall p \in \mathbb{N}\\
    E[(X_iX_i^*)^{p}] & \preccurlyeq \frac{R^{2p - 2}(2p)!}{2}W \hspace{3mm} \forall p \in \mathbb{N} \\
    \max \{||V||, ||W||\} &= \sigma^2
\end{align*}
where $V$ and $W$ are trace-class operators on $\mathcal{V}$ and $\mathcal{W}$, respectively. Let $\delta > 0$ and $\beta(\delta) = \log \Big(\frac{4\text{tr}(V + W)}{\delta\sigma^2}\Big)$. Then, we have that:
\begin{equation*}
    \Big|\Big|\frac{1}{N}\sum_{i = 1}^N X_i\Big|\Big| \leq \frac{4R\beta(\delta)}{N} + 2\sigma\sqrt{\frac{2\beta(\delta)}{N}}
\end{equation*}
\end{lemma}
with probability $1 - \delta$
\begin{proof}
We generalize the approach from \cite{tropp2015introduction} --- namely we define the ``dilation'' operator $T_i$ on $\mathcal{W} \times \mathcal{V}$ that maps $T_i: (w, v) \mapsto (X_i^{*}v, X_iw)$, where $X_i^{*}$ denotes the adjoint of $X_i$. Then, it is easy to see that $T_i$ is self-adjoint. Moreover, we have that $\Big|\Big|\sum_{i} T_i\Big|\Big| = \Big|\Big|\sum_{i} X_i\Big|\Big|$. Thus, we can apply Lemma \ref{Symmetric Bernstein} to the $T_i$. Indeed, observe that $T^2_i: (w,  v) \mapsto (X_iX_i^{*}w, X_i^{*}X_iv)$, and hence we have that $\mathbb{E}[T^{2p}_i]: (w,  v) \mapsto (\mathbb{E}[(XX^{*})^p]w, \mathbb{E}[(X^{*}X)^{p}]v)$. From this, we obtain that $\mathbb{E}[T_i^{2p}] \preccurlyeq \frac{R^{2p - 2}(2p)!}{2} U$, where $U: (w, v) \mapsto (Wv, Vv)$. Our result then follows from Lemma \ref{Symmetric Bernstein}.
\end{proof}
\begin{lemma}
\label{General Bernstein}
Let $\mathcal{H}_{K}$ and $\mathcal{H}_{L}$ be RKHSs on $\mathcal{X}$ and $\mathcal{Y}$, respectively. Let $X_{1}, X_{2}, \ldots X_{N}$ be i.i.d rank-1 operators from $\mathcal{H}_{K}$ to $\mathcal{H}_{L}$, with:
\begin{align*}
    E[(X_i^* X_i)^{p}] & \preccurlyeq \frac{R^{2p - 2}(2p)!}{2}V \hspace{3mm} \forall p \in \mathbb{N}\\
    E[(X_iX_i^*)^{p}] & \preccurlyeq \frac{R^{2p - 2}(2p)!}{2}W \hspace{3mm} \forall p \in \mathbb{N} \\
    \max \{||V||, ||W||\} &= \sigma^2
\end{align*}
where $V$ and $W$ are trace-class operators on $\mathcal{X}$ and $\mathcal{Y}$, respectively. Let $\delta > 0$ and $\beta(\delta) = \log \Big(\frac{4\text{tr}(V + W)}{\delta\sigma^2}\Big)$. Then, we have that:
\begin{equation*}
    \Big|\Big|\frac{1}{N}\sum_{i = 1}^N X_i - \mathbb{E}[X] \Big|\Big| \leq \frac{8R\beta(\delta)}{N} + 4\sigma\sqrt{\frac{2\beta(\delta)}{N}}
\end{equation*}
with probability $1 - \delta$
\end{lemma}
\begin{proof}
We observe that:
\begin{align}
    \mathbb{E}[((X_i - \mathbb{E}[X_i])^{*}(X_i - \mathbb{E}[X_i]))^p] & \preccurlyeq \mathbb{E}[||X_i - \mathbb{E}[X_i]||^{2(p - 1)}_{\text{HS}} (X_i - \mathbb{E}[X_i])^{*}(X_i - \mathbb{E}[X_i])] \nonumber \\
    & \preccurlyeq 2^{2p-1}\Big(\mathbb{E}[||X_i||^{2(p - 1)}_{\text{HS}} X_i^{*}X_i] + ||\mathbb{E}[X_i]||^{2(p-1)}_{\text{HS}} \mathbb{E}[X_i]^{*}\mathbb{E}[X_i]\Big) \label{eq: Convexity1} \\
    & \preccurlyeq 4^{p}\mathbb{E}[||X_i||^{2(p-1)}_{\text{HS}} X_i^{*}X_i] \label{eq: Convexity2} \\
    & = 4^{p}\mathbb{E}[(X_i^{*}X_i)^{p}] \nonumber
\end{align}
where \eqref{eq: Convexity1} and \eqref{eq: Convexity2} follow from the convexity of $||X||^{2(p-1)}_{\text{HS}}||Xf||^2_{\mathcal{Y}}$ for any $f \in \mathcal{H}_{K}$ by Lemma \ref{Convexity Result} and the definition of the semidefinite order, and the last step follows from the fact that $X_i$ is rank-1 and hence $||X_i||^{2p-2}_{\text{HS}} X_i^{*}X_i = (X_i^{*}X_i)^{p}$. We can show a similar conclusion for $\mathbb{E}[((X_i - \mathbb{E}[X_i])(X_i - \mathbb{E}[X_i])^{*})^p]$ from which the result follows by Lemma \ref{Rectangular Bernstein}. 
\end{proof}
\section{Auxiliary Results}
\label{Auxiliary Proofs}
\begin{lemma}
\label{Convexity Result}
Let $\mathcal{H}_1, \mathcal{H}_2$ be Hilbert spaces. Then, for any $y \in \mathcal{H}$, the functions $f: \mathcal{H} \to \mathbb{R}^{+}$ and $g: \mathcal{L}(\mathcal{H}_1, \mathcal{H}_2) \to \mathbb{R}^{+}$ given by $f(x) = ||x||_{H}^{2p}\langle y, x \rangle^2_{H}$ and $g(A) = ||A||^{2p}_{\text{HS}} ||Ay||^2_{\mathcal{H}_2}$ are convex for all $p \geq 1$ in the semidefinite order.
\end{lemma}
\begin{proof}
We must show that for any $x, z \in \mathcal{H}$ and $X, Z \in \mathcal{L}(\mathcal{H}_1, \mathcal{H}_2)$, the functions $\tilde{f}(t) = f(x + tz)$ and $\tilde{g}(t) = g(X + tZ)$ are convex in $t$. Observing that $\tilde{f}$ and $\tilde{g}$ can be expressed as: $\tilde{f}(t) = (||z||^2 t^2 + 2t \langle x, z \rangle + ||x||^2)^{p}(\langle x, z \rangle t +  \langle y, z \rangle)^2$ and $\tilde{g}(t) = (||Z||_{\text{HS}}^2 t^2 + 2\text{tr}(X^{*} \circ Z)t   + ||X||_{\text{HS}}^2)^{p}(||Xf||^2_{\mathcal{H}_2}t^2 + 2t\langle Xf, Zf \rangle_{\mathcal{H}_2} + ||Xf||^2_{\mathcal{H}_2})$, the claim can be readily verified by taking second derivatives.
\end{proof}
\begin{corollary}
\label{Matrix Convexity}
Let $\mathcal{H}$ be a Hilbert space, and let $\mathcal{L}_1(\mathcal{H})$ be the space of rank-1 linear operators on $\mathcal{H}$. Then, the operator-valued function $g: \mathcal{H} \to \mathcal{L}_1(\mathcal{H})$ given by and $g(u) = (u \otimes u)^{p}$ is convex for any $p \geq 1$.
\end{corollary}
\begin{proof}
Let $T \in \mathcal{L}_1(\mathcal{H})$. Then, we have that $T = u \otimes u$ for $u \in \mathcal{H}$. Thus, $f(T) = ||u||^{2p - 2}(u \otimes u)$. By the definition of the semidefinite order, we have $f$ is convex iff the real-valued function $f(T) = ||u||^{2p-2}\langle y, u \rangle^2$ for all $y \in \mathcal{H}$. The latter follows from Lemma \ref{Convexity Result}. 
\end{proof}

\begin{lemma}
\label{Convergent Series}
Suppose Assumption \ref{EVD} holds. Then, if $\beta > p$, there exists a constant $D > 0$ that depends only on $\beta$ and $p$, such that:
\begin{equation*}
    \sum_{i = 1}^{\infty} \Big(\frac{\mu^{\frac{\beta}{2}}_i}{\mu_i + \lambda}\Big)^2 \leq D\lambda^{\beta - p - 2}
\end{equation*}
\end{lemma}
\begin{proof}
We have that:
\begin{align}
    \sum_{i = 1}^{\infty} \Big(\frac{\mu^{\frac{\beta}{2}}_i}{\mu_i + \lambda}\Big)^2 & = \sum_{i = 1}^{\infty} \Big(\frac{\mu^{\frac{\beta}{2} - 1}_i}{1 + \lambda \mu^{-1}_i}\Big)^2 \nonumber \\
    & \leq \int_{0}^{\infty} \Big(\frac{(c^{-p} x)^{-\frac{p^{-1}\beta}{2} + p^{-1}}}{1 + \lambda (C^{-p} x)^{p^{-1}}}\Big)^2 dx \label{eq: Sub Eigendecay} \\
    &  = \lambda^{\beta - p - 2} \int_{0}^{\infty} \Big(\frac{(c^{-p}y)^{-\frac{p^{-1}\beta}{2} + p^{-1}}}{1 + (C^{-p}y)^{p^{-1}}}\Big)^2 dy \label{eq: Integral Sub}
\end{align}
where \eqref{eq: Sub Eigendecay} follows from Assumption \ref{EVD}, and \eqref{eq: Integral Sub} follows after making the substitution $\lambda^{p}x = y$. Now, observe that:
\begin{align}
    \int_{0}^{\infty} \Big(\frac{(c^{-p} y)^{-\frac{p^{-1}\beta}{2} + p^{-1}}}{1 + (C^{-p} y)^{p^{-1}}}\Big)^2 dy  & = \int_{0}^{\infty} \Big(\frac{(c^{-p}y)^{-\frac{p^{-1}\beta}{2}}}{(c/C) + (c^{-p} y)^{-p^{-1}}}\Big)^2 dy \nonumber \\
    & = \int_{0}^{1} \Big(\frac{(c^{-p}y)^{-\frac{p^{-1}\beta}{2}}}{(c/C) + (c^{-p} y)^{-p^{-1}}}\Big)^2 dy + \int_{1}^{\infty} \Big(\frac{(c^{-p}y)^{-\frac{p^{-1}\beta}{2}}}{(c/C) + (c^{-p} y)^{-p^{-1}}}\Big)^2 dy \nonumber \\
    & \leq \int_{0}^{1} \Big(\frac{c}{C}\Big)^{\frac{\beta}{2} - 1} dy + \frac{C^2}{c^{2 - \beta}} \int_{1}^{\infty} y^{-p^{-1}\beta} dy \label{eq: Apply A.1} \\
    & = D \nonumber \\
    & < \infty
\end{align}
where \eqref{eq: Apply A.1} follows from the fact that $\frac{\beta}{2} < 1$ and Lemma A.1 in \cite{fischer2020sobolev}, and the last line follows from $\beta > p$. 
\end{proof}
The following two results, which are from \cite{fischer2020sobolev} (and were originally discussed in \cite{steinwart2012mercer}), characterizes the boundedness of the kernel and the ``effective dimension''. We include them here for completeness . 
\begin{lemma}
\label{EMB to h-bound}
Suppose $||k^{\alpha}||_{\infty}  < \infty$. Then, we have that:
\begin{equation*}
    ||(C_{XX} + \lambda)^{-\frac{1}{2}}k(X, \cdot)||_{K} \leq \lambda^{-\frac{\alpha}{2}}||k^{\alpha}||_{\infty}
\end{equation*}
\end{lemma}
\begin{proof}
From definition, we have that:
\begin{align*}
    (C_{XX} + \lambda)^{-\frac{1}{2}}k(X, \cdot) &= \sum_{i} \sqrt{\frac{\mu_i}{\mu_i + \lambda}} \cdot e_i(x)(\mu^{\frac{1}{2}}_ie_i) \\
    & = \sum_{i} \sqrt{\frac{\mu^{1 - \alpha}_i}{\mu_i + \lambda}} \cdot \mu^{\frac{\alpha}{2}}e_i(x)(\mu^{\frac{1}{2}}_ie_i) 
\end{align*}
Thus:
\begin{equation*}
    ||(C_{XX} + \lambda)^{-\frac{1}{2}}k(X, \cdot)||_{K}^2 \leq \Big(\max_{i} \frac{\mu^{1 - \alpha}_i}{\mu_i + \lambda}\Big) \sum_{i} \mu^{\alpha}_i e^2_i(x) \leq \lambda^{-\alpha}||k^{\alpha}||^{2}_{\infty}
\end{equation*}
by Lemma $A.1$ in \cite{fischer2020sobolev}. 
\end{proof}
\begin{lemma}
\label{Effective Dimension Bound}
Suppose Assumption \ref{EVD} holds. Then, there exists a $M_1 > 0$ such that:
\begin{equation*}
    \mathcal{N}(\lambda) = \text{tr}\Big(C_{\nu}(C_{\nu} + \lambda)^{-1}\Big) \leq M_1\lambda^{-p}
\end{equation*}
\end{lemma}
\begin{proof}
See Lemma 6.3 in \cite{fischer2020sobolev}
\end{proof}
\begin{lemma}
\label{Effective Dimension Lower Bound}
Suppose Assumption \ref{EVD} holds. Then, there exists a $M_2 > 0$ such that:
\begin{equation*}
    \mathcal{N}(\lambda) \geq M_2\lambda^{-p}
\end{equation*}
\end{lemma}
\begin{proof}
We have that:
\begin{align*}
    \mathcal{N}(\lambda) & = \sum_{i = 1}^{\infty} \frac{\mu_i}{\mu_i + \lambda} \\
    & \geq \sum_{i = 1}^{\infty} \frac{ci^{-p^{-1}}}{Ci^{-p^{-1}} + \lambda} \\
    & \geq \int_{1}^{\infty} \frac{cx^{-p^{-1}}}{Cx^{-p^{-1}} + \lambda} dx \\
    & = \int_{1}^{\infty} \frac{c}{C + \lambda x^{p^{-1}}} dx \\
    & = \lambda^{-p}\int_{1}^{\infty} \frac{c}{C + y^{p^{-1}}} dy
\end{align*}
where the last line follows from making the substitution $y = \lambda^p x$. Then, our result follows from observing $\int_{1}^{\infty} \frac{c}{C + y^{p^{-1}}} = M_2 < \infty$ since $p < 1$ by assumption.
\end{proof}

\begin{lemma}
\label{Gaussian Interpolation Spaces include Constants}
Let $\mathcal{H}_K$ be a Gaussian RKHS over $\mathbb{R}^d$ with kernel $K(x, y) = \text{exp}\Big(-\frac{||x - t||^2}{\sigma^2}\Big)$ for some $\sigma > 0$. Then, for every $\beta \in (0, 1)$, $\mathcal{H}^{\beta}_K$ contains constant functions.
\end{lemma}
\begin{proof}
We only treat the one-dimensional case $d = 1$ and note that the more general case follows easily from the argument of \cite{steinwart2008support}. By \cite{minh2010some}, we have that:
\begin{equation*}
    \mathcal{H}_{K} = \Big\{f = e^{-\frac{x^2}{\sigma^2}}\sum_{k = 0}^{\infty} w_k x^{k}: ||f||_{K}^2 \equiv \sum_{k = 0}^{\infty} \frac{w^2_k \sigma^{2k} k!}{2^k} < \infty \Big\}
\end{equation*}
Thus, we have by definition:
\begin{equation*}
    \mathcal{H}^{\beta}_{K} = \Big\{f = e^{-\frac{x^2}{\sigma^2}}\sum_{k = 0}^{\infty} w_k x^{k}: ||f||_{\mathcal{H}^{\beta}}^2 \equiv \sum_{k = 0}^{\infty} \frac{w^2_k \sigma^{2\beta k} (k!)^{\beta}}{2^{\beta k}} < \infty \Big\}
\end{equation*}
For any $c \in \mathbb{R}$, we have that:
\begin{equation*}
    c = e^{-\frac{x^2}{\sigma^2}}  \cdot e^{\frac{x^2}{\sigma^2}} c = e^{-\frac{x^2}{\sigma^2}}  \cdot \sum_{k = 0}^{\infty} \frac{cx^{2k}}{k! \sigma^{2k}}
\end{equation*}
Thus, we may define $w_{2k} = \frac{c}{k! \sigma^{2k}}$ and $w_{2k+1} = 0$ for $k \in \mathbb{N}$. Therefore, we have:
\begin{align*}
    ||g_f||_{\mathcal{H}^{\beta}}^2 & = c^2\sum_{k = 0}^{\infty} \frac{\sigma^{4\beta k} ((2k)!)^{\beta}}{4^{\beta k}(k!)^2 \sigma^{4k}} \\
    & = c^2\sum_{k = 0}^{\infty} \frac{\sigma^{4(\beta - 1) k} ((2k)!)^{\beta}}{4^{\beta k}(k!)^2}
\end{align*}
Let $a_{k} = \frac{\sigma^{4(\beta - 1) k} ((2k)!)^{\beta}}{4^{\beta k}(k!)^2}$. Now applying Stirling's formula $n! \sim \sqrt{2\pi n} \Big(\frac{n}{e}\Big)^n$, we have (for sufficiently large $k$):
\begin{equation*}
    a_{k} \sim e^{2(1 - \beta)k}(\pi k)^{\frac{\beta - 2}{2}} \sigma^{4(\beta - 1) k} k^{2(\beta - 1)k} 
\end{equation*}
Noting that since $\beta < 1$, for $k \geq \frac{2^{\frac{1}{2(1- \beta)}}e}{\sigma^2}$, we have $e^{2(1 - \beta)k}(\pi k)^{\frac{\beta - 2}{2}} \sigma^{4(\beta - 1) k} k^{2(\beta - 1)k}  \leq \Big(\frac{1}{2}\Big)^{k}$. Thus, we have $||g_f||_{\mathcal{H}^{\beta}}^2 = c^2\sum_{k = 0}^{\infty} a_k < \infty$, and we obtain our result. 
\end{proof}
\end{document}